\documentclass[10pt]{article}

\usepackage[margin=1in]{geometry}
\usepackage{natbib}
\usepackage{parskip}
\usepackage[hyphens]{url}  
\usepackage{graphicx} 
\usepackage{caption} 
\usepackage{algorithm}
\usepackage{algorithmic}

\usepackage[most]{tcolorbox}

\newtcolorbox{graybox}{
  colback=gray!10,
  colframe=gray!10,
  boxrule=0pt,
  arc=2pt,
  left=2pt,
  right=2pt,
  top=2pt,
  bottom=2pt,
  breakable
}

\usepackage{newfloat}
\usepackage{listings}
\DeclareCaptionStyle{ruled}{labelfont=normalfont,labelsep=colon,strut=off} 
\floatstyle{ruled}
\newfloat{listing}{tb}{lst}{}
\floatname{listing}{Listing}

\usepackage{booktabs}

\usepackage{pgfplots}
\pgfplotsset{compat=1.18}

\usepackage[utf8]{inputenc}	
\usepackage{amsmath,amsthm,amsfonts,amssymb,amscd}
\usepackage{multirow,booktabs}
\usepackage[table]{xcolor}
\usepackage{algorithm}
\usepackage{algorithmic}
\usepackage{lastpage}
\usepackage{enumitem}
\usepackage{fancyhdr}
\usepackage{mathrsfs}
\usepackage{calc}
\usepackage{cancel}
\usepackage{tikz}
\usepackage[retainorgcmds]{IEEEtrantools}
\usepackage{amsmath}

\usepackage{empheq}
\usepackage{framed}
\usepackage[most]{tcolorbox}
\usepackage{xcolor}
\usepackage{tcolorbox}
\usepackage{thm-restate}
\usepackage{xcolor}
\colorlet{shadecolor}{orange!15}
\theoremstyle{definition}

\usepackage{amsmath}
\usepackage{amssymb}

\RequirePackage{xspace}

\newcommand{\algname}{\texttt{Adaptive Robust ETC}\@\xspace}
\newcommand{\algnameshort}{\texttt{AdaR-ETC}\@\xspace}

\usepackage{amsthm}

\theoremstyle{plain}

\newtheorem{theorem}{Theorem}
\newtheorem{lemma}[theorem]{Lemma}

\newtheorem{definition}[theorem]{Definition}
\usepackage{natbib}

\DeclareMathOperator*{\argmax}{arg\,max}

\title{Parameter-Free Heavy-Tailed Bandits}
\author{Gianmarco Genalti \\
    \texttt{gianmarco.genalti@polimi.it} \\
    Politecnico di Milano 
\and
Alberto Maria Metelli\\
\texttt{albertomaria.metelli@polimi.it} \\
    Politecnico di Milano 
\and
}
\date{July 2026}

\begin{document}

\setlength{\abovedisplayskip}{5pt}
\setlength{\belowdisplayskip}{5pt}
\setlength{\abovedisplayshortskip}{3pt}
\setlength{\belowdisplayshortskip}{3pt}
\setlength{\floatsep}{4pt}
\setlength{\textfloatsep}{5pt}

\maketitle

\begin{abstract}
Heavy-tailed distributions arise naturally in sequential decision-making problems such as financial investment, online advertising, and network management, where rare but extreme outcomes can dominate performance. Heavy-tailed bandits model online decision-making in these settings by assuming only that rewards $X$ satisfy $\mathbb{E}[|X|^{1+\epsilon}]\leq u$, for some tail exponent $\epsilon\in(0,1]$ and moment bound $u<+\infty$. However, most existing regret minimization algorithms require these parameters to be known. This assumption is particularly restrictive in practice: $\epsilon$ and $u$ govern the frequency and magnitude of rare events and are therefore precisely the quantities that are hardest to infer reliably from limited observations.

Motivated by an open problem posed by Genalti and Metelli at COLT 2025, we resolve the assumption-free adaptation problem for heavy-tailed bandits and characterize the price in the regret of not knowing the tail parameters. We first study adaptation to the moment bound $u$ for a fixed tail exponent $\epsilon$. We prove that every algorithm unaware of $u$, or of any upper bound on it, must obey a sharp trade-off between its distribution-dependent and distribution-free regret guarantees. We then introduce a scheduled-exploration algorithm that requires no knowledge of $u$ and matches the resulting adaptation frontier up to logarithmic factors. Finally, we show that the same algorithm can be instanced without knowing $\epsilon$ by calibrating its exploration schedule to the endpoint $\epsilon=1$. It achieves sublinear regret for every fixed $\epsilon>0$, while no algorithm can guarantee sublinear regret uniformly over all $\epsilon\in(0,1]$. Altogether, our results resolve the COLT open problem without additional distributional assumptions and provide a sharp characterization of the statistical cost of adapting to unknown heavy tails.
\end{abstract}

\section{Introduction}

Heavy-tailed rewards arise naturally in sequential decision-making problems
such as financial investment \citep{gagliolo2011algorithm, genalti2026catoni}, online advertising \citep{anderson2007long}, and network management \citep{liebeherr2012delay},
where rare but extreme observations may dominate performance. The
heavy-tailed stochastic multi-armed bandit model captures these settings by
assuming only that the rewards $X$ of every arm satisfy $\mathbb E[|X|^{1+\epsilon}]
    \le u$ for some $\epsilon\in(0,1]$ and $u<+\infty$. The parameter $\epsilon$, named \emph{tail exponent},
controls the heaviness of the tails, while $u$, named \emph{moment bound}, controls their scale. When
both parameters are known, robust estimators can be calibrated to attain the
distribution-free regret of
\begin{align}\label{eq:oracleBound}
\widetilde{\mathcal{O}}\big(u^{\frac{1}{1+\epsilon}}
    K^{\frac{\epsilon}{1+\epsilon}}
    T^{\frac{1}{1+\epsilon}}\big),
\end{align}    
where $K$ is the number of arms and $T$ is the horizon \citep{bubeck2013bandits}.

The knowledge of $(\epsilon,u)$ is particularly restrictive in the
real-world. Both parameters describe the behavior of
rare observations and are therefore difficult to infer reliably from
limited data. Moreover, misspecifying $\epsilon$ changes the polynomial
concentration rate of the estimators \citep{lugosi2019mean}, rather than merely its constants. Prior work showed
that the known-parameter guarantees cannot generally be recovered without
either paying an additional regret or imposing further distributional
assumptions \citep{genalti2024varepsilon}. This motivated the open problem of
\citet{genalti2025open} asking what are the best assumption-free regret guarantees
when the heavy-tail parameters are unknown and which algorithms can attain
them. 

\paragraph{Contributions.} In this paper, we first fix the tail exponent $\epsilon$ and study adaptation to the
unknown moment bound $u$. Rather than considering only the best distribution-free
rate, we characterize how robustness to arbitrary scales $u$ trades off with
performance on favorable instances. Let $\Phi_{free}(K,T)$ denote a
moment-free \emph{distribution-free} regret rate and let $\Phi_{dep}(K,T)$ denote the
gap-sum-normalized \emph{distribution-dependent} regret rate.\footnote{These quantities will be formally defined later in the paper.} We prove that every
strategy unaware of $u$ must satisfy (Theorem \ref{thr:adaptiveLB})
\begin{align}\label{eq:front}
    \Phi_{dep}(K,T)
    \Phi_{free}(K,T)^{\frac{1+\epsilon}{\epsilon}}
    =
    \Omega\big(
        T^{\frac{1+\epsilon}{\epsilon}}
    \big).
\end{align}
Thus, improving the distribution-free guarantee necessarily deteriorates
the distribution-dependent one. This establishes a \emph{frontier} that reveals how adaptation trades off distribution-dependent and distribution-free guarantees.

We complement this lower bound by proposing an adaptive regret minimization algorithm, \texttt{Adaptive Robust ETC}
(\algnameshort), which leverages Median-of-Means \citep{lugosi2019mean} Explore-Then-Commit \citep{lattimore2020bandit} strategy that does
not make use of the knowledge of $u$. To achieve adaptivity, \algnameshort is parametrized by  
    $\alpha
    \in
    \left[
        (1+\epsilon)/(1+2\epsilon),
        1
    \right)$, 
    $q
    \in
    \left[
        0,
        \epsilon/(1+2\epsilon)
    \right]$,
and $\beta_\alpha
    =
    (1-\alpha)(1+\epsilon)/\epsilon$,
and obtains a distribution-free regret bound (Theorem \ref{thr:adaptiveUBfree}),
\begin{align}
    \Phi_{free}(K,T)
    =
    \widetilde{\mathcal O}\big(
        K^{\frac{\epsilon}{1+\epsilon}(1-q)}
        T^\alpha
    \big)
\end{align}
and a distribution-dependent regret bound (Theorem \ref{thr:adaptiveUBdep}),
\begin{align}
    \Phi_{dep}(K,T)
    =
    \widetilde{\mathcal O}\big(
        K^{q-1}T^{\beta_\alpha}
    \big).
\end{align}
The combination of these guarantees is tight on the joint lower bound frontier of Equation \eqref{eq:front}.
Balancing
the horizon dependence and the one on the number of arms gives
\begin{align}\label{eq:boundEpsKnown}
    \widetilde{\mathcal O}\big(
        u^{\frac{1}{1+\epsilon}}
        K^{\frac{\epsilon}{1+2\epsilon}}
        T^{\frac{1+\epsilon}{1+2\epsilon}}
    \big).
\end{align}

This also represents the best possible distribution-free guarantee that can be obtained by any algorithm unaware of $u$. 
As visible from the exponent of $T$, this regret bound is worse compared to the one of Equation \eqref{eq:oracleBound}, establishing the price of adaptivity.
Specializing this bound in the finite variance case ($\epsilon=1$), we obtain a $\widetilde{\mathcal{O}}(T^\frac{2}{3})$ rate, strictly greater than the $\widetilde{\mathcal{O}}(\sqrt{T})$ rate attainable in bandits with bounded or subgaussian rewards \citep{lattimore2020bandit}.

We then remove the knowledge of $\epsilon$. The Median-of-Means estimator uses neither $\epsilon$ nor $u$; only the exploration schedule makes use of $\epsilon$. Calibrating it to the known endpoint $\epsilon=1$ yields a single
strategy, independent of both parameters, satisfying a distribution-free regret bound 
\begin{align}
    \widetilde{\mathcal O}\big(
        u^{\frac{1}{1+\epsilon}}
        K^{\frac{2\epsilon}{3(1+\epsilon)}}
        T^{\frac{3+\epsilon}{3(1+\epsilon)}}
    \big)
\end{align}
for every fixed true $\epsilon>0$ (Theorem \ref{thr:adaptive_unknown_parameters}). Thus matches the bound of Eq. \eqref{eq:boundEpsKnown} for $\epsilon = 1$, but deteriorates for all $\epsilon \in (0,1)$. Furthermore, we prove a pairwise lower bound (Theorem \ref{thm:unknown_epsilon_frontier}) across
moment orders showing that this profile is optimal, up to logarithmic
factors, among strategies retaining the balanced finite-variance guarantee
at $\epsilon=1$. Hence, no single policy is optimal at every moment order;
adaptation is described by a frontier rather than by one oracle curve.

Finally, our pointwise guarantee cannot be made uniform. Although
the regret is sublinear for every fixed $\epsilon>0$, no strategy can
guarantee sublinear regret (normalized by $u$) uniformly over
$\epsilon\in(0,1]$ (Corollary \ref{cor:uniform_epsilon_impossibility}). Indeed, as $\epsilon$ approaches zero, the finite-moment
assumption becomes arbitrarily weak and the exponent of $T$ approaches
one. 

Altogether, our results characterize the assumption-free cost of
$u$-adaptivity, provide an algorithm attaining the
resulting distribution-dependent frontier, and identify the limits of
simultaneous $(\epsilon,u)$-adaptation.

\section{Heavy-Tailed Bandits}
We recall some fundamental notions on heavy-tailed bandits \citep{bubeck2013bandits} and the required notations for regret rates defined in \citep{hadiji2020adaptation}.

\paragraph{Interaction Protocol.}
In the stochastic multi-armed bandit problem
\citep{lattimore2020bandit}, a learner interacts with $K\in\mathbb N_{\ge2}$
arms for a horizon of $T\in\mathbb N$ rounds. A bandit instance is an
ordered tuple $\underline\nu=(\nu_1,\ldots,\nu_K)$ of probability distributions on $\mathbb R$. For every arm $i\in[K]$,\footnote{Given $k\in\mathbb{N}$, we define $[k] \coloneqq \{1, 2, \ldots,k\}$.}
successive pulls produce an i.i.d.\ sequence
$X_{i,1},X_{i,2},\ldots
    \stackrel{\mathrm{i.i.d.}}{\sim}
    \nu_i$,
and the reward sequences are independent across arms.

At every round $t\in[T]$, the learner selects an arm $I_t\in[K]$ according
to a possibly randomized rule $\pi$ measurable w.r.t. the history up to
$t-1$. The learner then observes the reward generated by the
selected arm $X_{I_t,t}$. Let
$
    N_i(t)
    =
    \sum_{s=1}^t
    \mathbf{1}\{I_s=i\}
$
be the number of pulls of arm $i$ up to $t$.

\paragraph{Heavy-Tailed Bandits.}
In this paper, we consider heavy-tailed reward distributions. For
$\epsilon\in(0,1]$ and $u>0$, we denote  the set of heavy-tailed bandit instances as:
\[
    \mathcal H_{\epsilon,u}
    =
    \left\{
        (\nu_1,\ldots,\nu_K):
        \mathbb E_{X\sim\nu_i}
        \left[
            |X|^{1+\epsilon}
        \right]
        \le u,
        \;
        \forall i\in[K]
    \right\}.
\]
The dependence of $\mathcal H_{\epsilon,u}$ on $K$ is kept implicit in the
notation.

Let $\mu_i \coloneqq
    \mathbb E_{X\sim\nu_i}[X]$ be the expected reward of arm $i$, $
    \mu^*
    \coloneqq
    \max_{i\in[K]}\mu_i$ be the optimal expected reward, and $
    \Delta_i
    \coloneqq
    \mu^*-\mu_i$ be the suboptimality gap of arm $i$.
The expected cumulative regret of a strategy $\pi$ is given by
\[
    R_T^\pi(\underline\nu)
    =
    \mathbb E_{\underline\nu,\pi}
    \left[
        \sum_{t=1}^T
        \left(
            \mu^*-\mu_{I_t}
        \right)
    \right]
    =
    \sum_{i=1}^K
    \Delta_i
    \mathbb E_{\underline\nu,\pi}[N_i(T)],
\]
where the expectation is taken over both the randomness rewards and the internal
randomness of the strategy. When the strategy is clear from the context,
we write $R_T(\underline\nu)$.

\paragraph{$(\epsilon,u)$-adaptivity in Heavy-Tailed Bandits.} Most of the existing algorithms require, as an input, both $u$ and $\epsilon$ (see, e.g., \citet{bubeck2013bandits, agrawal2020optimal, lee2022minimax}). These parameters govern the behavior of the tails of the reward distributions and cannot be estimated reliably  \citep{bahadur1956nonexistence}. Since heavy-tailed bandits model complex real-world scenarios beyond the canonical yet limiting distributional assumptions, requiring such knowledge severely limits their scope. A recent research stream \citep{ashutosh2021bandit, genalti2024varepsilon, tamas2024data, chen2024uniinf} focused on devising $(\epsilon,u)$-adaptive algorithms, \emph{i.e.}, unaware of the values of $u$ and/or $\epsilon$, and on characterizing the statistical limits of learnability without this knowledge.
\citet{genalti2024varepsilon} show that adaptivity comes at a cost: either that knowledge is substituted by another structural assumption, or the same regret bounds as if $u$ and/or $\epsilon$ are known (\textit{oracle rates}) cannot be achieved. 

\begin{graybox}
\citet{genalti2025open} propose a COLT open problem addressing the
following questions:
\begin{enumerate}
    \item \emph{What are the best possible regret rates that can be
    achieved under adaptivity requirements?}

    \item \emph{What algorithms match such rates?}

    \item \emph{Is there a \textit{best} assumption that allows for
    oracle rates?}
\end{enumerate}

\textbf{In this paper, we provide a collection of results that,
altogether, answer the first two questions.}
\end{graybox}

\section{Regret Rates in Heavy-Tailed Bandits}
In the stochastic bandit literature, there are two main ways to express regret
guarantees: \emph{distribution-free} bounds and \emph{distribution-dependent} bounds.\footnote{With little approximation, these guarantees are also known in the literature as \emph{worst-case} and \emph{instance-dependent}.} In
this paper, the latter term refers specifically to the gap-sum-normalized
coefficient introduced in Definition~\ref{def:phi_dep}.
In distribution-dependent bounds, the guarantee depends on the specific
instance through the suboptimality gaps $\Delta_i$,
whereas distribution-free bounds
remove this dependence by considering the worst case over the entire class.
We recall the known lower bound which characterizes the minimax regret in
heavy-tailed bandits when both $u$ and $\epsilon$ are known to the learner.

\begin{restatable}[Distribution-free regret lower bound,
\citet{bubeck2013bandits}]{theorem}{StandardLowerBound}
\label{thr:standard_lb}
Fix $\epsilon\in(0,1]$. There exists a constant
$c_{\epsilon}>0$, depending only on $\epsilon$, such that,
for every $u>0$, every $K\ge2$, every horizon
$T\ge K$, and every exploration strategy $\pi$,
\begin{equation}
    \sup_{\underline\nu\in\mathcal H_{\epsilon,u}}
    R_T^\pi(\underline\nu)
    \ge
    c_{\epsilon}
    u^{\frac{1}{1+\epsilon}}
    K^{\frac{\epsilon}{1+\epsilon}}
    T^{\frac{1}{1+\epsilon}}.
    \label{eq:standard_distribution_free_lb}
\end{equation}
\end{restatable}

In this paper, we tackle $(\epsilon,u)$-adaptivity through a two-step
approach. First, we characterize the $u$-adaptive setting, in which
$\epsilon$ is known. We then remove the knowledge of
$\epsilon$ and quantify the additional difficulty of simultaneously
adapting to both parameters.
It is worth noting that the he
$u$-adaptive setting is of interest on its own. Indeed, in non-heavy-tailed MABs,
adaptation to an unknown bound on the support of the rewards has been
characterized in \citet{hadiji2020adaptation}. However, in heavy-tailed bandits the limits of
adaptation to an unknown moment bound remain open.

Inspired by the definitions of
\citet{hadiji2020adaptation} for bounded-support bandits, we now define the
two types of regret rates considered in our analysis. 

\begin{definition}[Moment-free distribution-free regret rate]
\label{def:phi_free}
A strategy $\pi$ for stochastic heavy-tailed bandits admits a moment-free
distribution-free regret bound $\Phi_{free}$
if, without knowing $u$, it guarantees
\begin{align}
    R_T^\pi(\underline\nu)
    \le
    u^{\frac{1}{1+\epsilon}}
    \Phi_{free}(K,T),
    \label{eq:definition_phi_free}
\end{align}
for all $K\ge2$, $ T\ge1$, $ u>0$, and $\underline\nu\in\mathcal H_{\epsilon,u}$.
\end{definition}

Theorem~\ref{thr:standard_lb} implies that every attainable moment-free
distribution-free rate, whenever $T\ge K$, satisfies 
\begin{equation}
    \Phi_{free}(K,T)
    \ge
    c_{\mathrm{std},\epsilon}
    K^{\frac{\epsilon}{1+\epsilon}}
    T^{\frac{1}{1+\epsilon}},
    \label{eq:necessary_phi_free_rate}
\end{equation}
for some constant $c_{\mathrm{std},\epsilon}>0$ possibly depending on $\epsilon$.

\begin{definition}[Distribution-dependent regret rate]
\label{def:phi_dep}
A strategy $\pi$ for stochastic heavy-tailed bandits admits a
distribution-dependent rate $\Phi_{dep}$
if, without knowing $u$, it guarantees
\begin{equation}
    \limsup_{T\rightarrow+\infty}
    \frac{
        R_T^\pi(\underline\nu)
    }{
        \Phi_{dep}(K,T)
    }
    \le
    \sum_{i:\Delta_i>0}\Delta_i,
    \label{eq:definition_phi_dep}
\end{equation}
for all $K\ge2$, $u>0$, $\underline\nu\in\mathcal H_{\epsilon,u}$.
\end{definition}

The \emph{normalization} in Definition~\ref{def:phi_dep} entails no loss of
generality. Indeed, any multiplicative constant in a
distribution-dependent upper bound can be absorbed into the definition of
$\Phi_{dep}(K,T)$. This is necessary to compare rates
and to state the adaptation frontier with a constant that depends only on
$\epsilon$.
The term distribution-dependent rate has a specific meaning in this paper:
$\Phi_{dep}(K,T)$ is the coefficient multiplying the sum of the
suboptimality gaps. It should not be confused with the classical
distribution-dependent bounds, which may exhibit a different dependence
on the gaps.

The functions $\Phi_{free}(K,T)$ and $\Phi_{dep}(K,T)$ explicitly depend on both the horizon $T$ and the number of arms $K$. Their dependence
on $\epsilon$ is suppressed because $\epsilon$ is fixed throughout the
$u$-adaptive analysis, whereas neither rate is allowed to depend on the
unknown value of $u$. In the following sections, we characterize the
trade-off between these two rates and study their optimal dependence on
$K$, $T$, and $\epsilon$.
\section{$u$-adaptivity: \textit{You can't have it both ways}}
\label{sec:sec_3}

In this section, we characterize the limits of learnability when the moment
bound $u$ is unknown. Our main result establishes a fundamental trade-off
between the distribution-free and the distribution-dependent guarantees.
These two guarantees cannot be optimized independently: improving one
necessarily deteriorates the other. Moreover, the trade-off concerns both the
dependence on the horizon $T$ and the dependence on the number of arms $K$. The following result shows that the
two quantities must lie on a frontier.
\begin{restatable}[Existence of a trade-off]{theorem}{adaptiveLB}
\label{thr:adaptiveLB}
Fix $\epsilon\in(0,1]$. Consider a strategy that does not know $u$ and
admits a moment-free distribution-free rate
$\Phi_{free}(K,T)=o(T)$, for every fixed $K\ge2$. Then, any
distribution-dependent rate $\Phi_{dep}(K,T)$ satisfying
Definition~\ref{def:phi_dep} fulfills
\begin{equation}
    \liminf_{T\rightarrow+\infty}
    \frac{
        \Phi_{dep}(K,T)
        \Phi_{free}(K,T)^{\frac{1+\epsilon}{\epsilon}}
    }{
        T^{\frac{1+\epsilon}{\epsilon}}
    }
    \ge
    c_\epsilon,
    \quad
    \forall K\ge2,
    \label{eq:u_adaptivity_frontier}
\end{equation}
where $c_\epsilon>0$ depends only on $\epsilon$.
\end{restatable}
The proof follows the change-of-measure procedure developed in
\citep{hadiji2020adaptation}, together with the instance
construction of \citet{genalti2024varepsilon}. 

Theorem~\ref{thr:adaptiveLB} establishes a frontier rather than two independent
lower bounds. Intuitively, a learner that aggressively pursues a small
distribution-dependent regret explores apparently suboptimal arms only a
limited number of times. In the heavy-tailed setting, however, an arm that
mostly returns low-reward observations may still hide a rare but extremely large
reward. Protecting against these alternatives requires additional
exploration. This improves the distribution-free guarantee, but it is
unnecessary on favorable instances and deteriorates the
distribution-dependent performance.

The tension first appears in the dependence on $T$ and, among strategies
lying on the optimal horizon frontier, also in the dependence on $K$.
To isolate the exponents, suppose that the two rates admit
monomial envelopes of the form $\Phi_{dep}(K,T)=K^aT^c$ and $\Phi_{free}(K,T)=K^bT^d$,
up to multiplicative factors that are bounded above and below by constants
independent of $K$ and $T$. Substituting these expressions into
Theorem~\ref{thr:adaptiveLB} gives
\[
    \liminf_{T\rightarrow+\infty}
    K^{a+\frac{1+\epsilon}{\epsilon}b}
    T^{
        c+\frac{1+\epsilon}{\epsilon}d
        -\frac{1+\epsilon}{\epsilon}
    }
    \ge
    c_\epsilon.
\]
Consequently, the exponents of $T$ must satisfy $c+d(1+\epsilon)/\epsilon
    \ge (1+\epsilon)/\epsilon$.
This inequality describes the fundamental
trade-off in the horizon $T$. Decreasing the distribution-free exponent $d$
forces the distribution-dependent exponent $c$ to increase, and vice
versa. In Figure \ref{fig:lower_bound}, we provide a graphical representation of this trade-off. 

The dependence on $K$ requires additional care because
Theorem~\ref{thr:adaptiveLB} takes $T$ to infinity for each fixed $K$.
If the inequality holds strictly, polynomial
growth in $T$ may compensate for any fixed dependence on $K$. Consider
instead strategies attaining the horizon boundary $c+d(1+\epsilon)/\epsilon
    = (1+\epsilon)/\epsilon$.
For these strategies, the dependence on $T$ cancels in the lower bound.
Since $c_\epsilon$ is independent of $K$, the exponent of $K$ must
then satisfy $a+b(1+\epsilon)/\epsilon\ge0$.
Thus, among strategies lying on
the optimal horizon frontier, improving the distribution-free dependence
on $K$ necessarily deteriorates the distribution-dependent dependence.
\begin{figure}[t]
    \centering
    \resizebox{.5\linewidth}{!}{\begin{tikzpicture}
\footnotesize
\begin{axis}[
    width=6cm,
    height=4.5cm,
    xmin=0,
    xmax=1,
    ymin=0,
    ymax=1,
    axis lines=left,
    clip=false,
    xlabel={Distribution-free exponent $d$},
ylabel={Distribution-dependent\\exponent $c$},
xlabel style={font=\footnotesize},
ylabel style={
    font=\footnotesize,
    align=center
},
    tick label style={font=\footnotesize},
    xtick={
        0,
        0.6666667,
        0.75,
        0.8333333,
        1
    },
    xticklabels={
        $0$,
         $\frac{2}{3}$,
        $\frac{3}{4}$,
        $\frac{5}{6}$,
        $1$
    },
    ytick={
        0,
        0.6666667,
        0.75,
        0.8333333,
        1
    },
    yticklabels={
        $0$,
        $2/3$,
        $3/4$,
        $5/6$,
        $1$
    },
    legend style={
        at={(0.02,0.03)},
        anchor=south west,
        draw=none,
        fill=none,
        font=\small,
        cells={anchor=west}
    },
]

\addplot[
    blue,
    densely dotted,
    thin,
    forget plot
]
coordinates {
    (0,0.6666667)
    (0.6666667,0.6666667)
    (0.6666667,0)
};

\addplot[
    orange,
    densely dotted,
    thin,
    forget plot
]
coordinates {
    (0,0.75)
    (0.75,0.75)
    (0.75,0)
};

\addplot[
    green!60!black,
    densely dotted,
    thin,
    forget plot
]
coordinates {
    (0,0.8333333)
    (0.8333333,0.8333333)
    (0.8333333,0)
};

\addplot[
    blue,
    very thick
]
coordinates {
    (0,1)
    (0.5,1)
    (1,0)
};
\addlegendentry{$\varepsilon=1$}

\addplot[
    orange,
    very thick
]
coordinates {
    (0,1)
    (0.6666667,1)
    (1,0)
};
\addlegendentry{$\varepsilon=0.5$}

\addplot[
    green!60!black,
    very thick
]
coordinates {
    (0,1)
    (0.8,1)
    (1,0)
};
\addlegendentry{$\varepsilon=0.25$}

\addplot[
    only marks,
    mark=*,
    mark size=2pt,
    blue,
    forget plot
]
coordinates {
    (0.6666667,0.6666667)
};

\addplot[
    only marks,
    mark=*,
    mark size=2pt,
    orange,
    forget plot
]
coordinates {
    (0.75,0.75)
};

\addplot[
    only marks,
    mark=*,
    mark size=2pt,
    green!60!black,
    forget plot
]
coordinates {
    (0.8333333,0.8333333)
};

\end{axis}
\end{tikzpicture}}
    \vspace{-.3cm}\caption{Trade-off between the distribution-dependent and the distribution-free rates in $T$.}
    \label{fig:lower_bound}
\end{figure}
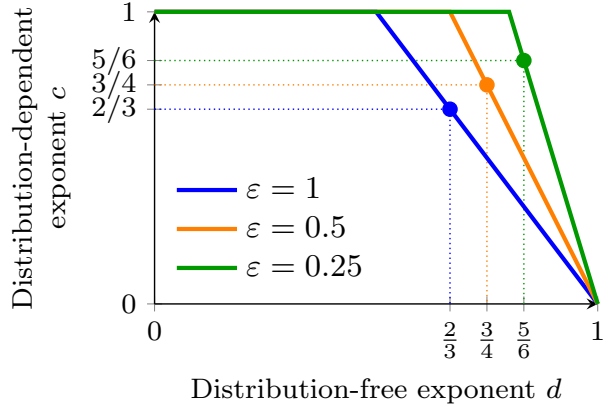
\paragraph{Equal-gap interpretation.}
Consider an equal-gap instance consisting of one optimal arm and $K-1$
suboptimal arms, each with gap $\Delta>0$. On this family,
$\sum_{i:\Delta_i>0}\Delta_i=(K-1)\Delta$. Therefore, if
$\Phi_{dep}(K,T)=K^aT^c$, then
\[
    R_T(\underline\nu)
    \le
    K^aT^c(K-1)\Delta
    \le
    K^{a+1}T^c\Delta.
\]
Thus, the exponent governing the dependence of the actual regret on $K$ is
$a+1$, rather than $a$. Rewriting the inequality in terms of the 
distribution-dependent exponent gives
$(a+1)+b(1+\epsilon)/\epsilon\ge1$. This formulation clarifies that both the
distribution-free regret and the actual distribution-dependent regret may
deteriorate as $K$ increases. The tension is not that one quantity must
decrease while the other increases. Rather, their  exponents in
$K$ cannot both be made arbitrarily small. On the boundary of the $K$-frontier, we have
$a=-b(1+\epsilon)/\epsilon$, and hence the 
distribution-dependent exponent is
$a+1=1-b(1+\epsilon)/\epsilon$. Therefore, reducing the
distribution-free exponent $b$ necessarily increases the 
distribution-dependent one $a+1$.

\paragraph{Representative points for $\epsilon=1$.}
At the finite-variance endpoint, the $K$-frontier becomes $a+2b\ge0$.
Table~\ref{tab:K_frontier_examples} reports two representative points on
its boundary. The first choice yields a distribution-dependent regret that is essentially
independent of $K$ on the equal-gap family, but pays a $K^{1/2}$
distribution-free factor. Moving to the second point improves the
distribution-free dependence from $K^{1/2}$ to $K^{1/3}$, while the
equal-gap distribution-dependent regret deteriorates from a constant
dependence on $K$ to $K^{1/3}$.
\begin{table}[t]
    \centering
    \small
    \setlength{\tabcolsep}{4pt}
    \caption{Representative points on the $K$-frontier for $\epsilon=1$.}
    \label{tab:K_frontier_examples}
    \begin{tabular}{lccc}
        \toprule
        Operating point
        &
        $(b,a)$
        &
        $\Phi_{free}(K,T)$
        &
        Equal-gap $R_T$
        \\
        \midrule
        Instance-oriented
        &
        $(1/2,-1)$
        &
        $K^{1/2}T^d$
        &
        $\mathcal O(T^c\Delta)$
        \\
        $K$-balanced
        &
        $(1/3,-2/3)$
        &
        $K^{1/3}T^d$
        &
        $\mathcal O(K^{1/3}T^c\Delta)$
        \\
        \bottomrule
    \end{tabular}
\end{table}
\paragraph{Balancing both $K$ and $T$.}
A natural operating point is obtained by requiring the two guarantees to
have the same polynomial dependence on both the horizon and the number of
arms. For the horizon, imposing $c=d$ gives
$c=d\ge(1+\epsilon)/(1+2\epsilon)$. Thus, when the two guarantees are
required to have the same dependence on $T$, neither exponent can be
smaller than $(1+\epsilon)/(1+2\epsilon)$. For the number of arms, the distribution-free exponent is $b$, whereas the
 distribution-dependent exponent on the equal-gap family is $a+1$.
Balancing them amounts to imposing $b=a+1$. Combining this identity with
the boundary condition $a+b(1+\epsilon)/\epsilon=0$ gives $b=\frac{\epsilon}{1+2\epsilon}$ and $a=-\frac{1+\epsilon}{1+2\epsilon}$.
At the point balancing both $K$ and $T$, the two rates have form 
$
    \Phi_{free}(K,T)
    =
    K^{\frac{\epsilon}{1+2\epsilon}}
    T^{\frac{1+\epsilon}{1+2\epsilon}}$ and $
    \Phi_{dep}(K,T)
    =
    K^{-\frac{1+\epsilon}{1+2\epsilon}}
    T^{\frac{1+\epsilon}{1+2\epsilon}}.
$
On the equal-gap family, this corresponds to $R_T(\underline\nu)
    =
    \mathcal O\big(
        K^{\frac{\epsilon}{1+2\epsilon}}
        T^{\frac{1+\epsilon}{1+2\epsilon}}
        \Delta
    \big)$,
which has the same dependence on $K$ and $T$ as the
distribution-free rate.

For $\epsilon=1$, the balanced frontier point is
$\Phi_{free}(K,T)=K^{1/3}T^{2/3}$ and
$\Phi_{dep}(K,T)=K^{-2/3}T^{2/3}$. In particular, the balanced horizon
dependence is $T^{2/3}$, which is worse than the $\sqrt{T}$ dependence
arising in adaptation to an unknown bounded reward range
\citep{hadiji2020adaptation}. This deterioration reflects the additional
difficulty of ruling out rare and arbitrarily large rewards, typical of heavy-tailed distributions, when only a
finite, unknown moment bound is available.

\section{Explore-Then-Commit suffices for $u$-adaptivity}
\label{sec:sec_4}

In this section, we propose an algorithm, fully unaware of $u$, that achieves
regret guarantees that are tight on the frontier defined by
Theorem~\ref{thr:adaptiveLB}. The algorithm allows us to select a point on
both the $T$-frontier and the $K$-frontier. The dependence on $T$ is controlled
by the parameter $\alpha$, whereas the one on $K$ is controlled by an
additional exploration parameter $q$.

\begin{algorithm}[t]
\caption{\algname~~(\algnameshort)}
\label{alg:adar_etc}
\small
\begin{algorithmic}[1]
\REQUIRE Number of arms $K$, horizon $T$, exploration parameters
$\alpha\in[(1+\epsilon)/(1+2\epsilon),1)$ and
$q\in[0,\epsilon/(1+2\epsilon)]$.
    \STATE Set $\beta_\alpha \gets \frac{(1-\alpha)(1+\epsilon)}{\epsilon}$, $B_T \gets \left\lceil 8\log(KT^3)\right\rceil$, $\widetilde L_T \gets KB_T+\left\lceil K^qT^{\beta_\alpha}\right\rceil$, $L_T \gets \min\{T,\widetilde L_T\}$, $\mathcal F_i\gets\emptyset$ for all $i \in [K]$.

    \FOR{$t=1,\ldots,L_T$}
        \STATE Select arm $I_t\gets 1+((t-1)\bmod K)$.
        \STATE Observe the reward and append it to $\mathcal F_{I_t}$.
    \ENDFOR

    \IF{$L_T<T$}
        \FOR{$i\in[K]$}
            \STATE Compute $\widehat\mu_i^{MoM}(\mathcal F_i)$.
        \ENDFOR
        \STATE Select
        $\widehat I^*\in\argmax_{i\in[K]}\widehat\mu_i^{MoM}(\mathcal F_i)$.
        \FOR{$t=L_T+1,\ldots,T$}
            \STATE Select arm $I_t\gets\widehat I^*$.
        \ENDFOR
    \ENDIF
\end{algorithmic}
\end{algorithm}

Surprisingly, the algorithm is very simple and natural. In fact, an
Explore-Then-Commit (ETC) strategy with robust estimation and a tuned amount
of exploration $L_T$ is enough to get there. We call our algorithm \algname
(\algnameshort, for short), and we report its pseudocode in
Algorithm~\ref{alg:adar_etc}. In the next paragraphs, we describe the main components of \algnameshort.

\paragraph{Robust Estimator.} Since the distributions are heavy-tailed, the empirical mean is not a
suitable estimator \citep{bubeck2013bandits}. We then resort to the well-known
Median of Means estimator (MoM, for short). Let $\mathcal F_i
    =
    \{X_{i,1},\ldots,X_{i,n_i}\}$
be the set of $n_i$ exploration samples collected from arm $i$. We divide
these samples into $B_T$ blocks of equal size $s_i
    =
    \left\lfloor
        n_i/B_T
    \right\rfloor$.
For every $b\in[B_T]$, we define the $b$-th block as $\mathcal G_{i,b}
    =
    \left\{
        X_{i,(b-1)s_i+1},
        \ldots,
        X_{i,bs_i}
    \right\}$.
Thus, each block contains exactly $s_i$ samples. If $n_i$ is not divisible
by $B_T$, the remaining $n_i-B_Ts_i$ samples are discarded.

For each block $\mathcal G_{i,b}$, we define the corresponding block average
as $\overline X_{i,b}
    =
    \frac{1}{s_i}
    \sum_{\ell=(b-1)s_i+1}^{bs_i}
    X_{i,\ell}$ for $ b\in[B_T]$.
Let $\overline X_{i,(1)}
    \le
    \overline X_{i,(2)}
    \le
    \cdots
    \le
    \overline X_{i,(B_T)}$
denote the ordered block averages. The MoM estimator is defined
as
\[
    \widehat{\mu}^{MoM}_i(\mathcal F_i)
    =
    \overline X_{i,\left(\left\lceil B_T/2\right\rceil\right)}.
\]

Intuitively, although a single block average
may be corrupted by an extreme observation, under the finite
$(1+\epsilon)$-moment assumption, a constant fraction of the block averages
remains close to the true mean with high probability. Taking their median prevents a small number of atypical blocks from significantly
affecting the estimate.

More precisely, if
$\mathbb E[|X|^{1+\epsilon}]\le u$, the estimation error is, with
high probability, of order $u^{\frac{1}{1+\epsilon}}
    {B_T}^{\frac{\epsilon}{1+\epsilon}}{n_i}^{-\frac{\epsilon}{1+\epsilon}}$.
Thus, we choose $B_T$ logarithmic in $K$ and $T$.
Most importantly, while $u$ and $\epsilon$ determine the rate appearing in
the concentration analysis, the computation of the MoM estimator itself does
not require knowledge of $u$ nor $\epsilon$. This estimator enjoys optimal, up to constants,
concentration properties around the true mean
\citep{bubeck2013bandits}.

\paragraph{Exploration Budget.} The exploration budget of \algnameshort is controlled by two parameters.
The parameter $\alpha$ determines how the exploration budget scales with the
horizon, through
$
    \beta_\alpha
    =
    (1-\alpha)(1+\epsilon)/\epsilon.
$
A larger $\alpha$ corresponds to a smaller $\beta_\alpha$ and thus to
less exploration as $T$ grows. This improves the distribution-dependent
rate on $T$, at the cost of a worse distribution-free rate.
The parameter $q$ plays the analogous role for the dependence on the number
of arms. The polynomial part of the total exploration budget is $K^qT^{\beta_\alpha}$.
Since exploration is performed in a round-robin fashion, each arm gets
approximately $K^{q-1}T^{\beta_\alpha}$
samples. Increasing $q$ assigns more exploration samples to each
arm as $K$ grows. This improves the distribution-free rate on $K$, but
increases the regret paid on favorable instances.

The restrictions on $\alpha$ and $q$ are chosen so that the exploration
contribution does not dominate the estimation one. Indeed, $\alpha\ge\beta_\alpha$
is equivalent to $\alpha
    \ge (1+\epsilon)/(1+2\epsilon)$,
whereas $q
    \le \epsilon/(1+2\epsilon)$
is equivalent to $q
    \le
    (1-q)\epsilon/(1+\epsilon)$.
\paragraph{Regret Guarantees.} The following results formalize the resulting trade-off and certify
the tightness of \algnameshort with respect to the frontier of
Theorem~\ref{thr:adaptiveLB}.

\begin{restatable}[Distribution-free regret of \algnameshort]{theorem}{adaptiveUBfree}
\label{thr:adaptiveUBfree}
Let $\epsilon\in(0,1]$ be fixed and known. Let
$\alpha\in[(1+\epsilon)/(1+2\epsilon),1)$ and
$q\in[0,\epsilon/(1+2\epsilon)]$.
For every $u>0$ and every instance
$\underline\nu\in\mathcal H_{\epsilon,u}$, \algnameshort satisfies
\[
    R_T^{\text{\algnameshort}}(\underline\nu)
    \le
    \widetilde{\mathcal O}\big(
        u^{\frac{1}{1+\epsilon}}
        K^{\frac{\epsilon}{1+\epsilon}(1-q)}
        T^\alpha
    \big),
\]
where $\widetilde{\mathcal O}$ hides polylogarithmic terms in $T$ and constants depending only on
$\epsilon$, $\alpha$, and $q$.
\end{restatable}

The two main contributions to the regret are, up to logarithmic factors, $u^{\frac{1}{1+\epsilon}}
    K^qT^{\beta_\alpha}$
due to exploration, and $u^{\frac{1}{1+\epsilon}}
    K^{\frac{\epsilon}{1+\epsilon}(1-q)}
    T^\alpha$
due to committing according to the MoM estimates. The restrictions imposed
on $\alpha$ and $q$ ensure that the latter term dominates. Thus, \algnameshort is $u$-adaptive with distribution-free rate $\Phi_{free}(K,T)
    =
    \widetilde{\mathcal O}\big(
        K^{\frac{\epsilon}{1+\epsilon}(1-q)}
        T^\alpha
    \big)$.

On the other hand, we have the following distribution-dependent guarantee.

\begin{restatable}[Distribution-dependent regret of \algnameshort]{theorem}{adaptiveUBdep}
\label{thr:adaptiveUBdep}
Let $\epsilon\in(0,1]$ be fixed and known. Let $\alpha\in[(1+\epsilon)/(1+2\epsilon),1)$ and
$q\in[0,\epsilon/(1+2\epsilon)]$.
For every fixed instance
$\underline\nu\in\mathcal H_{\epsilon,u}$, \algnameshort satisfies
\[
    \limsup_{T\to+\infty}
    \frac{
        R_T^{\text{\algnameshort}}(\underline\nu)
    }{
        K^{q-1}T^{\beta_\alpha}
    }
    \le
    \sum_{i:\Delta_i>0}\Delta_i.
\]
\end{restatable}

Thus, \algnameshort is $u$-adaptive with
distribution-dependent rate
    $\Phi_{dep}(K,T)
    =
    \widetilde{\mathcal O}\big(
        K^{q-1}T^{\beta_\alpha}
    \big)$.

The two parameters $\alpha$ and $q$ control two distinct, but
parallel, trade-offs. The parameter $\alpha$ determines the trade-off in the
horizon $T$: $\Phi_{free}(K,T)
    \propto
    T^\alpha$ and
    $\Phi_{dep}(K,T)
    \propto
    T^{\beta_\alpha}$.
Since $\alpha
    +
    \beta_\alpha \epsilon/(1+\epsilon)
    =
    1$,
the two exponents lie exactly on the $T$-frontier.
Similarly, the parameter $q$ determines the trade-off in the number of arms $K$: $\Phi_{free}(K,T)
    \propto
    K^{\frac{\epsilon}{1+\epsilon}(1-q)}$
    and 
    $\Phi_{dep}(K,T)
    \propto
    K^{q-1}$.
These exponents satisfy $(1-q)\epsilon/(1+\epsilon)
    +
    (q-1)\epsilon/(1+\epsilon)
    =
    0$.
Thus, at the level of exponents, the choice of $q$ realizes
the equality case of the $K$-trade-off associated with the optimal
$T$-frontier.

Hence, for every admissible choice of $\alpha$ and $q$, \algnameshort\
matches the $T$-frontier of Theorem~\ref{thr:adaptiveLB} up to logarithmic
factors. Its explicit dependence on $K$ realizes the corresponding
polynomial $K$-trade-off on the horizon-optimal boundary.
\section{Characterizing $(\epsilon,u)$-adaptivity}
\label{sec:sec_5}
We now remove the knowledge of $\epsilon$ and consider a single strategy
that uses neither $u$ nor $\epsilon$. This setting involves two distinct
adaptation constraints. The first is the frontier associated with the
unknown scale $u$, characterized in Theorem~\ref{thr:adaptiveLB}. The
second is a new frontier across different moment orders $\epsilon$: improving the
regret guarantee on a lighter-tailed class necessarily worsens the
guarantee on heavier-tailed classes.

\paragraph{$(\epsilon,u)$-adaptive \algnameshort.} We first construct an order-free version of \algnameshort\ by
calibrating both its $T$-dependence and its $K$-dependence to the
finite-variance endpoint $\epsilon=1$. We then show that, for every fixed
$\epsilon>0$, the resulting distribution-free and distribution-dependent
guarantees lie on the unknown-$u$ frontier. Finally, we prove that its
distribution-free guarantee is also tight, up to logarithmic factors,
among strategies retaining the optimal endpoint guarantee at $\epsilon=1$. At the finite-variance endpoint $\epsilon=1$, the choice balancing the
dependence on the horizon $T$ is $\alpha=2/3$ and $\beta_\alpha=2/3$,
whereas the choice balancing the distribution-free and the 
distribution-dependent dependence on the number of arms $K$ is $q=1/3$.
Thus, we define the order-free version of \algnameshort\ by setting
directly $B_T
    =
    \left\lceil
        8\log(KT^3)
    \right\rceil$
and $\widetilde L_T
    =
    KB_T
    +
    \left\lceil
        K^{1/3}T^{2/3}
    \right\rceil$.
As in the previous section, the effective exploration budget is $L_T=\min\{T,\widetilde L_T\}$
and the arms are explored in round-robin order. The resulting strategy is fully unaware of both $u$ and
$\epsilon$.

The following theorem characterizes its regret.

\begin{restatable}[Regret of \algnameshort calibrated with $\epsilon = 1$]
{theorem}{adaptiveUnknownParameters}
\label{thr:adaptive_unknown_parameters}
Let $\epsilon\in(0,1]$ and $u>0$ be fixed. For
every $T\ge K\ge2$ and every instance
$\underline\nu\in\mathcal H_{\epsilon,u}$, the order-free version of
\algnameshort\ calibrated with $\epsilon= 1$ satisfies
\begin{equation}
    R_T^{\text{\algnameshort}}(\underline\nu)
    \le
    \widetilde{\mathcal O}\left(
        u^{\frac{1}{1+\epsilon}}
        K^{\frac{2\epsilon}{3(1+\epsilon)}}
        T^{\frac{3+\epsilon}{3(1+\epsilon)}}
    \right).
    \label{eq:unknown_parameters_free_rate}
\end{equation}
where
$\widetilde{\mathcal O}$ hides factors at most polylogarithmic in $K$ and
$T$ and constants depending on $\epsilon$.
Moreover, for every fixed instance
$\underline\nu\in\mathcal H_{\epsilon,u}$,
\begin{equation}
    \limsup_{T\to+\infty}
    \frac{
        R_T^{\text{\algnameshort}}(\underline\nu)
    }{
        K^{-2/3}T^{2/3}
    }
    \le
    \sum_{i:\Delta_i>0}\Delta_i.
    \label{eq:unknown_parameters_dep_rate}
\end{equation}
\end{restatable}

The distribution-free guarantee follows from the same
exploration--estimation decomposition  of
Theorem~\ref{thr:adaptiveUBfree}. For every fixed $K$,
the regret is sublinear on every fixed heavy-tailed moment class, even
though the algorithm does not know the value of $\epsilon$.
The distribution-dependent guarantee follows because, on every fixed
instance, the probability of committing to a suboptimal arm vanishes
sufficiently fast. Asymptotically, the regret is therefore entirely due to
round-robin exploration. 
For every fixed $\epsilon$, Theorem~\ref{thr:adaptive_unknown_parameters}
therefore gives the rates $\Phi_{free,\epsilon}(K,T)
    =
    \widetilde{\mathcal O}\big(
        K^{\frac{2}{3}\frac{\epsilon}{1+\epsilon}}
        T^{1-\frac{2}{3}\frac{\epsilon}{1+\epsilon}}
    \big)$
and $\Phi_{dep,\epsilon}(K,T)
    =
    \widetilde{\mathcal O}\big(
        K^{-2/3}T^{2/3}
    \big)$.
Consequently,
\begin{align*}
    &
    \Phi_{free,\epsilon}(K,T)
    \Phi_{dep,\epsilon}(K,T)^{\frac{\epsilon}{1+\epsilon}}
    \\
    &\qquad=
    \widetilde{\mathcal O}\big(
        K^{
            \frac{2}{3}\frac{\epsilon}{1+\epsilon}
            -
            \frac{2}{3}\frac{\epsilon}{1+\epsilon}
        }
        T^{
            1-\frac{2}{3}\frac{\epsilon}{1+\epsilon}
            +
            \frac{2}{3}\frac{\epsilon}{1+\epsilon}
        }
    \big)
    =
    \widetilde{\mathcal O}(T).
\end{align*}
Thus, for every fixed $\epsilon>0$, the order-free version of
\algnameshort\ matches the unknown-$u$ frontier in its polynomial
dependence on $T$. The exponents of $K$ satisfy the corresponding trade-off on the horizon-optimal boundary.

\paragraph{Limits of $(\epsilon,u)$-adaptivity.} If $\epsilon$ were known, the point balancing both the $T$-dependence and
the $K$-dependence on the $u$-adaptivity frontier would be obtained by
choosing $\alpha^\star(\epsilon)
    =
    (1+\epsilon)/(1+2\epsilon)$
and $q^\star(\epsilon)
    =\epsilon/(1+2\epsilon)$. 
For fixed $K$, the price of not knowing $\epsilon$ is therefore the
difference
\begin{align}
    \frac{3+\epsilon}{3(1+\epsilon)}
    -
    \frac{1+\epsilon}{1+2\epsilon}
    &=
    \frac{
        \epsilon(1-\epsilon)
    }{
        3(1+\epsilon)(1+2\epsilon)
    }
    \ge0.
    \label{eq:unknown_epsilon_price}
\end{align}
The two exponents coincide at $\epsilon=1$, whereas the lack of knowledge
of $\epsilon$ causes a strictly positive loss in the dependence on $T$ for
every $\epsilon\in(0,1)$ (Figure \ref{fig:lower_bound}).

\begin{figure}[t]
    \centering
    \resizebox{0.7\linewidth}{!}{\begin{tikzpicture}
\small
\begin{axis}[
  width=4.5cm,
  height=4.5cm,
  xlabel={$\varepsilon$},
  ylabel={$T$ exponent},
  xmin=0,
  xmax=1.05,
  ymin=0.4,
  ymax=1.03,
  domain=0.001:1,
  samples=200,
  grid=none,
  ytick={0.4,0.5,0.6,0.6666667,0.8,1},
  yticklabels={
    ,
    $1/2$,
    ,
    $2/3$,
    ,
    $1$
  },
  legend style={
    at={(1.03,0.5)},
    anchor=west,
    draw=none,
    row sep=6pt
  },
]

  \addplot[
    gray,
    thin,
    forget plot
  ]
  coordinates {(0,0.5) (1.05,0.5)};

   \addplot[
    gray,
    thin,
    forget plot
  ]
  coordinates {(0,1) (1.05,1)};

  \addplot[
    gray,
    thin,
    forget plot
  ]
  coordinates {(0,0.6666667) (1.05,0.6666667)};

  \addplot[
    ultra thick,
    green!60!black
  ]
  {1/(1+x)};
  \addlegendentry{
    \footnotesize
    $\frac{1}{1+\varepsilon}$
    ($(\varepsilon,u)$ known)
  }

  \addplot[
    ultra thick,
    blue,
    densely dashed
  ]
  {(1+x)/(1+2*x)};
  \addlegendentry{
    \footnotesize
    $\frac{1+\varepsilon}{1+2\varepsilon}$
    ($u$ unknown, $\varepsilon$ known)
  }

  \addplot[
    ultra thick,
    red,
    densely dotted
  ]
  {(3+x)/(3*(1+x))};
  \addlegendentry{
    \footnotesize
    $\frac{3+\varepsilon}{3(1+\varepsilon)}$
    ($(\varepsilon,u)$ unknown)
  }

  \addplot[
    only marks,
    mark=o,
    mark size=2.5pt,
    thick,
    black,
    forget plot
  ]
  coordinates {(0,1)};

  \addplot[
    only marks,
    mark=*,
    mark size=2.5pt,
    green!60!black,
    forget plot
  ]
  coordinates {(1,{1/2})};

  \addplot[
    only marks,
    mark=*,
    mark size=2.5pt,
    black,
    forget plot
  ]
  coordinates {(1,{2/3})};

\end{axis}
\end{tikzpicture}}\caption{$T$ exponent as a function of $
    \epsilon$.}
    \label{fig:Texp}
\end{figure}
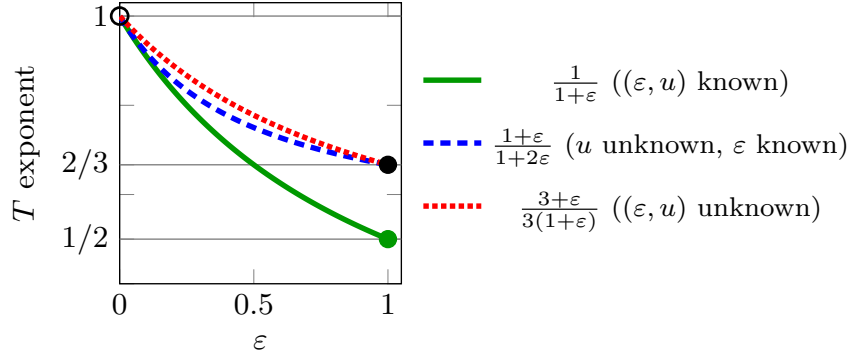

When the dependence on $K$ is also retained, the two rates are not
overall comparable. Indeed,
\[
    \frac{2\epsilon}{3(1+\epsilon)}
    -
    \frac{\epsilon}{1+2\epsilon}
    =
    -
    \frac{
        \epsilon(1-\epsilon)
    }{
        3(1+\epsilon)(1+2\epsilon) \le 0
    }.
\]
Thus, the endpoint-calibrated strategy has a worse dependence on $T$, but
a smaller distribution-free exponent in $K$.

This behavior is a consequence of the $K$-trade-off characterized in the
previous sections. For every $\epsilon<1$, we have $
    q^\star(\epsilon)
    =
    \epsilon / (1+2\epsilon)
    < 1/3$.
The order-free strategy therefore explores more aggressively in $K$ than
the strategy designed with knowledge of $\epsilon$. This additional
exploration improves the distribution-free dependence on $K$, but worsens the
dependence on $K$ on favorable instances.

We now show that this redistribution is unavoidable. Fix an exploration strategy $\pi$ that uses neither $u$ nor $\epsilon$. For an instance $\underline\nu$, define its intrinsic moment scale at
order $\epsilon$ as
\begin{equation}
    U_\epsilon(\underline\nu)
    \coloneqq
    \max_{i\in[K]}
    \left(
        \mathbb E_{X\sim\nu_i}
        \left[
            |X|^{1+\epsilon}
        \right]
    \right)^{\frac{1}{1+\epsilon}}.
    \label{eq:intrinsic_scale_unknown_parameters}
\end{equation}
The normalized regret profile of $\pi$ is
\begin{equation}
    \Phi_\epsilon^\pi(K,T)
    \coloneqq
    \sup_{
        \underline\nu:
        0<U_\epsilon(\underline\nu)<+\infty
    }
    \frac{
        R_T^\pi(\underline\nu)
    }{
        U_\epsilon(\underline\nu)
    }.
    \label{eq:normalized_profile_unknown_parameters}
\end{equation}
The normalized profile is equivalent to a scale-uniform raw-moment
guarantee. More precisely, for every $B\ge0$, $\Phi_\epsilon^\pi(K,T)\le B$
if and only if the same strategy $\pi$ satisfies $\sup_{\underline\nu\in\mathcal H_{\epsilon,u}}
    R_T^\pi(\underline\nu)
    \le
    u^{\frac{1}{1+\epsilon}}B$
simultaneously for every $u>0$. Indeed, every
$\underline\nu\in\mathcal H_{\epsilon,u}$ satisfies
$U_\epsilon(\underline\nu)\le u^{1/(1+\epsilon)}$, while the reverse
implication follows by setting
$u=U_\epsilon(\underline\nu)^{1+\epsilon}$.

To simplify the notation, once the strategy $\pi$ is fixed, we write
$\Phi_\epsilon(K,T)$ in place of $\Phi_\epsilon^\pi(K,T)$. 

\begin{restatable}[Pairwise lower bound for adaptation to an unknown $\epsilon$]
{theorem}{unknownepsilonfrontier}
\label{thm:unknown_epsilon_frontier}
There exists a numerical constant $c_1>0$ such that, for every fixed
strategy whose action rule uses neither the realized moment order nor the moment bound, $T\ge K\ge2$, and
$0<\epsilon\le\epsilon'\le1$, if
$\Phi_{\epsilon'}(K,T)\le T/4$, then
\begin{equation}
    \Phi_\epsilon(K,T)
    \Phi_{\epsilon'}(K,T)^{
        \frac{\epsilon}{1+\epsilon}
    }
    \ge
    c_1
    T
    K^{
        \frac{\epsilon}{1+\epsilon}
    }.
    \label{eq:pairwise_unknown_epsilon_frontier}
\end{equation}
\end{restatable}

Specializing Theorem~\ref{thm:unknown_epsilon_frontier} to
$\epsilon'=1$ gives a conditional tightness result. More precisely, among
strategies retaining the endpoint guarantee $\Phi_1(K,T)
    \le \widetilde{\mathcal{O}}\left(K^{1/3}T^{2/3}\right)$,
the pairwise frontier forces the dependence on both $K$ and $T$ displayed
in Theorem~\ref{thr:adaptive_unknown_parameters}. This does not define an
unconditional minimax curve over all values of $\epsilon$: a different
strategy may deliberately accept a worse guarantee at $\epsilon=1$ to
improve its performance at another moment order.
Suppose that a strategy satisfies, in the non-saturated regime, $\Phi_1(K,T)
    \le
    \widetilde{\mathcal{O}}\big(
    K^{1/3}
    T^{2/3}\big)$.
Theorem~\ref{thm:unknown_epsilon_frontier} then gives
\begin{align}
    &\Phi_\epsilon(K,T)
    \ge 
    \widetilde{\Omega}\big(
    2^{-\frac{\epsilon}{1+\epsilon}}
    K^{\frac{2}{3}\frac{\epsilon}{1+\epsilon}}
    T^{1-\frac{2}{3}\frac{\epsilon}{1+\epsilon}}\big),
\label{eq:unknown_parameters_matching_lower_bound}
\end{align}
matching the
distribution-free upper bound of
Theorem~\ref{thr:adaptive_unknown_parameters} in both $K$ and $T$, up to
logarithmic factors.

Therefore, among strategies retaining the endpoint guarantee $\widetilde{\mathcal O}\left(
        K^{1/3}T^{2/3}
    \right)$
at $\epsilon=1$, the distribution-free rate of the order-free version of
\algnameshort\ is frontier-optimal, up to logarithmic factors, in its
joint dependence on $K$ and $T$.

The choice $q=1/3$ can also be recovered directly from this matching
requirement. Suppose that the endpoint schedule used a generic exponent
$q$. Its endpoint distribution-free rate would have order $K^{\frac{1-q}{2}}T^{2/3}$.
The pairwise lower bound would then imply, at moment order $\epsilon$, a
$K$-dependence of at least $K^{\frac{1+q}{2}\frac{\epsilon}{1+\epsilon}}$.
On the other hand, the corresponding Median of Means upper bound would
scale as $K^{\frac{\epsilon}{1+\epsilon}(1-q)}$.
The two exponents coincide if and only if
$(1-q)\epsilon/(1+\epsilon)
    =
    (1+q)\epsilon/(2+2\epsilon)$,
which gives $q=1/3$.
Thus, the $K^{1/3}T^{2/3}$ exploration schedule is not merely a convenient
endpoint choice: it is the unique polynomial schedule in this family that
matches the unknown-order frontier in $K$ and
$T$.

Finally, pointwise sublinearity cannot be strengthened to a uniform
guarantee over all moment orders.

\begin{restatable}[Impossibility of uniform sublinear adaptation]
{corollary}{uniformepsilonimpossibility}
\label{cor:uniform_epsilon_impossibility}
There exists a numerical constant $c_2>0$ such that, for every strategy
that uses neither $u$ nor $\epsilon$ and every $T\ge K\ge2$,
\begin{equation}
    \sup_{\epsilon\in(0,1]}
    \Phi_\epsilon(K,T)
    \ge
    c_2T.
    \label{eq:uniform_unknown_epsilon_impossibility}
\end{equation}
\end{restatable}

Hence, no strategy can guarantee sublinear normalized regret uniformly
over $\epsilon\in(0,1]$ without further assumptions.
There is no contradiction between this impossibility result and
Theorem~\ref{thr:adaptive_unknown_parameters}. The theorem fixes
$\epsilon>0$ and $K$ before letting $T$ grow, whereas the supremum in
Corollary~\ref{cor:uniform_epsilon_impossibility} may select a different
value of $\epsilon$ for every horizon. Consistently, $\lim_{\epsilon\rightarrow0}
    (3+\epsilon)/(3(1+\epsilon))
    =
    1$.
Thus, sublinear regret is achievable for every fixed $\epsilon>0$ and
fixed $K$, but not uniformly over the entire range
$\epsilon\in(0,1]$.

\paragraph{Why Calibrating to $\epsilon=1$?}
The choice of calibrating \algnameshort\ to the finite-variance endpoint
$\epsilon=1$ may appear arbitrary, especially because the same construction
can be calibrated to any design order
$\bar\epsilon\in(0,1]$. The parameter $\bar\epsilon$ should not be
interpreted as an estimate of the unknown true order $\epsilon$. Rather, it
selects an operating point on the adaptation frontier. Calibrating to
$\bar\epsilon=1$ is a natural choice because it requires no
additional information and preserves the optimal guarantee on the
finite-variance class. Since every admissible true order satisfies
$\epsilon\le1$, this choice always corresponds to an \emph{optimistic}
calibration: unless $\epsilon=1$, the algorithm overestimates the moment
order and explores less than a strategy calibrated to the true class. The
resulting deterioration is the price required to retain the
finite-variance guarantee.

More generally, fix a calibration order $\bar\epsilon\in(0,1]$ and define $
    \bar q
    =
    \bar\epsilon/(1+2\bar\epsilon)$ and $
    \bar\beta
    =
    (1+\bar\epsilon)/(1+2\bar\epsilon)$.
The calibrated version of \algnameshort\ uses the exploration budget $\widetilde L_T(\bar\epsilon)
    =
    KB_T
    +
    \left\lceil
        K^{\bar q}T^{\bar\beta}
    \right\rceil$.
It depends on $\bar\epsilon$, but not on
$\epsilon$ or on $u$.

\begin{restatable}[Regret of \algnameshort calibrated with $\epsilon = \overline{\epsilon}$]
{theorem}{adaptiveCalibration}
\label{thr:adaptive_calibration}
Fix a calibration order $\bar\epsilon\in(0,1]$. Let $\epsilon\in(0,1]$ and $u>0$ be fixed and unknown. For
every $T\ge K\ge2$ and every instance
$\underline\nu\in\mathcal H_{\epsilon,u}$, the $\bar\epsilon$-calibrated version of
\algnameshort\ calibrated with $\overline{\epsilon}$ satisfies
\begin{equation*}
    R_T(\underline\nu)
    \le
    \widetilde{\mathcal O}\left(
        u^{\frac{1}{1+\epsilon}}
        \begin{cases}
        K^{
            \frac{
                \epsilon(1+\bar\epsilon)
            }{
                (1+\epsilon)(1+2\bar\epsilon)
            }
        }
        T^{
            
            \frac{
                1+2\bar\epsilon+\epsilon\bar\epsilon
            }{
                (1+\epsilon)(1+2\bar\epsilon)
            }
        },
        &   
        \epsilon\le\bar\epsilon,
        \\[3mm]
        K^{\frac{\bar\epsilon}{1+2\bar\epsilon}}
        T^{\frac{1+\bar\epsilon}{1+2\bar\epsilon}},
        &
        \epsilon\ge\bar\epsilon.
        \end{cases}
    \right).
    \label{eq:arbitrary_calibration_profile}
\end{equation*}
Moreover, for every fixed instance
$\underline\nu\in\mathcal H_{\epsilon,u}$,
\begin{equation}
    \limsup_{T\to+\infty}
    \frac{
        R_T(\underline\nu)
    }{
        K^{-\frac{1+\bar\epsilon}{1+2\bar\epsilon}}
        T^{\frac{1+\bar\epsilon}{1+2\bar\epsilon}}
    }
    \le
    \sum_{i:\Delta_i>0}\Delta_i.
    \label{eq:arbitrary_calibration_dep}
\end{equation}
\end{restatable}
At the calibration order $\epsilon=\bar\epsilon$, the two branches coincide and give $\widetilde{\mathcal O}\big(
        u^{\frac{1}{1+\bar\epsilon}}
        K^{\frac{\bar\epsilon}{1+2\bar\epsilon}}
        T^{\frac{1+\bar\epsilon}{1+2\bar\epsilon}}
    \big)$,
namely the balanced unknown-$u$ rate associated with $\bar\epsilon$.

The two sides of the calibration order have different interpretations. If
$\bar\epsilon<\epsilon$, the learner underestimates the moment order and
therefore explores more than necessary. This choice is conservative: the
exploration term dominates, and the regret remains at the rate associated
with $\bar\epsilon$. If $\bar\epsilon>\epsilon$, the learner
overestimates the moment order and explores too little for the true
heavy-tailed class. The estimation term then dominates, and the regret
deteriorates as the true $\epsilon$ decreases.

The pairwise lower bound in
Theorem~\ref{thm:unknown_epsilon_frontier} shows that these two branches
cannot be improved, up to logarithmic factors, while preserving the
balanced guarantee at $\bar\epsilon$. For $\epsilon<\bar\epsilon$, this
follows by applying the lower bound to the pair
$(\epsilon,\bar\epsilon)$; for $\epsilon>\bar\epsilon$, it follows by
applying it to $(\bar\epsilon,\epsilon)$. Thus, each calibration
selects a frontier-optimal profile across moment classes. In particular,
calibrating to $\bar\epsilon=1$ does not make the strategy simultaneously
minimax-optimal at every $\epsilon$, which is impossible. Instead, it
selects the frontier-optimal profile among strategies retaining the
balanced finite-variance guarantee.

\section{Conclusions and Future Directions}

We resolved the assumption-free rate and algorithmic components of the open problem of \citet{genalti2025open}. We characterized the regret frontier induced by adaptation to the unknown moment bound $u$, provided an algorithm, \algname, matching it up to logarithmic factors, and extended the analysis to the case in which both $u$ and $\epsilon$ are unknown. Two natural directions remain open. First, it would be interesting to design an anytime version of our algorithm; a doubling-trick construction should preserve the polynomial rates, at the cost of additional logarithmic factors. Second, completing the third part of the open problem requires identifying the weakest additional assumption under which the oracle rates can be recovered.
\clearpage
\addtocontents{toc}{\vspace{2em}}
\bibliographystyle{aaai2027}
\bibliography{aaai2027}

@article{lugosi2019mean,
  title={Mean estimation and regression under heavy-tailed distributions: A survey},
  author={Lugosi, G{\'a}bor and Mendelson, Shahar},
  journal={Foundations of Computational Mathematics},
  volume={19},
  number={5},
  pages={1145--1190},
  year={2019},
  publisher={Springer}
}

@book{anderson2007long,
  title={The long tail: How endless choice is creating unlimited demand},
  author={Anderson, Chris},
  year={2007},
  publisher={Random House}
}

@article{hadiji2020adaptation,
  title={Adaptation to the Range in $ K $-Armed Bandits},
  author={Hadiji, H{\'e}di and Stoltz, Gilles},
  journal={arXiv preprint arXiv:2006.03378},
  year={2020}
}

@inproceedings{agrawal2020optimal,
  title={Optimal  $\delta$-Correct Best-Arm Selection for Heavy-Tailed Distributions},
  author={Agrawal, Shubhada and Juneja, Sandeep and Glynn, Peter},
  booktitle={Algorithmic Learning Theory},
  pages={61--110},
  year={2020},
  organization={PMLR}
}

@book{lattimore2020bandit,
  title={Bandit algorithms},
  author={Lattimore, Tor and Szepesv{\'a}ri, Csaba},
  year={2020},
  publisher={Cambridge University Press}
}

@inproceedings{ashutosh2021bandit,
  title={Bandit algorithms: Letting go of logarithmic regret for statistical robustness},
  author={Ashutosh, Kumar and Nair, Jayakrishnan and Kagrecha, Anmol and Jagannathan, Krishna},
  booktitle={International Conference on Artificial Intelligence and Statistics},
  pages={622--630},
  year={2021},
  organization={PMLR}
}

@article{bubeck2013bandits,
  title={Bandits with heavy tail},
  author={Bubeck, S{\'e}bastien and Cesa-Bianchi, Nicolo and Lugosi, G{\'a}bor},
  journal={IEEE Transactions on Information Theory},
  volume={59},
  number={11},
  pages={7711--7717},
  year={2013},
  publisher={IEEE}
}

@inproceedings{genalti2024varepsilon,
  title={$(\varepsilon, u) $-Adaptive Regret Minimization in Heavy-Tailed Bandits},
  author={Genalti, Gianmarco and Marsigli, Lupo and Gatti, Nicola and Metelli, Alberto Maria},
  booktitle={The Thirty Seventh Annual Conference on Learning Theory},
  pages={1882--1915},
  year={2024},
  organization={PMLR}
}

@inproceedings{genalti2025open,
  title={Open Problem: Regret Minimization in Heavy-Tailed Bandits with Unknown Distributional Parameters},
  author={Genalti, Gianmarco and Metelli, Alberto Maria},
  booktitle={The Thirty Eighth Annual Conference on Learning Theory},
  pages={1--5},
  year={2025},
  organization={PMLR}
}

@article{gagliolo2011algorithm,
  title={Algorithm portfolio selection as a bandit problem with unbounded losses},
  author={Gagliolo, Matteo and Schmidhuber, J{\"u}rgen},
  journal={Annals of Mathematics and Artificial Intelligence},
  volume={61},
  pages={49--86},
  year={2011},
  publisher={Springer}
}

@article{liebeherr2012delay,
  title={Delay bounds in communication networks with heavy-tailed and self-similar traffic},
  author={Liebeherr, J{\"o}rg and Burchard, Almut and Ciucu, Florin},
  journal={IEEE Transactions on Information Theory},
  volume={58},
  number={2},
  pages={1010--1024},
  year={2012},
  publisher={IEEE}
}

@inproceedings{genalti2026catoni,
  title={Catoni-Style Change Point Detection for Regret Minimization in Piecewise-Stationary Heavy-Tailed Bandits},
  author={Genalti, Gianmarco and Bhatt, Sujay and Gatti, Nicola and Metelli, Alberto Maria},
  booktitle={The 29th International Conference on Artificial Intelligence and Statistics},
  year={2026}
}

@article{chen2024uniinf,
  title={uniINF: Best-of-both-worlds algorithm for parameter-free heavy-tailed MABs},
  author={Chen, Yu and Huang, Jiatai and Dai, Yan and Huang, Longbo},
  journal={arXiv preprint arXiv:2410.03284},
  year={2024}
}

@article{tamas2024data,
  title={Data-driven upper confidence bounds with near-optimal regret for heavy-tailed bandits},
  author={Tam{\'a}s, Ambrus and Szentp{\'e}teri, Szabolcs and Cs{\'a}ji, Bal{\'a}zs Csan{\'a}d},
  journal={arXiv preprint arXiv:2406.05710},
  year={2024}
}

@article{lee2022minimax,
  title={Minimax optimal bandits for heavy tail rewards},
  author={Lee, Kyungjae and Lim, Sungbin},
  journal={IEEE Transactions on Neural Networks and Learning Systems},
  volume={35},
  number={4},
  pages={5280--5294},
  year={2022},
  publisher={IEEE}
}

@article{bahadur1956nonexistence,
  title={The nonexistence of certain statistical procedures in nonparametric problems},
  author={Bahadur, Raghu R and Savage, Leonard J},
  journal={The Annals of Mathematical Statistics},
  volume={27},
  number={4},
  pages={1115--1122},
  year={1956},
  publisher={JSTOR}
}
\clearpage
\appendix
\section{Proof of the Lower Bound for $u$-adaptive Heavy-Tailed Bandits}
\adaptiveLB*
\begin{proof}
Let
\begin{equation*}
    \rho
    \coloneqq
    \frac{\epsilon}{1+\epsilon},
    \qquad
    p
    \coloneqq
    \frac{1}{\rho}
    =
    \frac{1+\epsilon}{\epsilon},
\end{equation*}
and let
\begin{equation*}
    \Phi
    \coloneqq
    \Phi_{free}(K,T).
\end{equation*}

Fix $K\ge2$ and $\Delta>0$, and consider the deterministic instance
$\underline\nu^{(0)}$ defined by
\begin{equation*}
    \nu^{(0)}_1=\delta_\Delta,
    \qquad
    \nu^{(0)}_i=\delta_0,
    \quad i\in\{2,\ldots,K\}.
\end{equation*}
Arm $1$ is the unique optimal arm, while every arm $i\ge2$ has gap
$\Delta_i=\Delta$. Consequently,
\begin{equation*}
    R_T(\underline\nu^{(0)})
    =
    \Delta
    \sum_{i=2}^K
    \mathbb E_{\underline\nu^{(0)}}[N_i(T)],
\end{equation*}
where
\begin{equation*}
    N_i(T)
    \coloneqq
    \sum_{t=1}^T
    \mathbf{1}\{I_t=i\}.
\end{equation*}

Since $\Phi_{free}(K,T)=o(T)$ for every fixed $K$, for all sufficiently
large $T$ it holds that
\begin{equation*}
    \Phi
    \le
    \frac{2^{-\rho}}{16}T.
\end{equation*}
For such values of $T$, define
\begin{equation*}
    \beta
    \coloneqq
    \left(
        \frac{16\Phi}{T}
    \right)^p.
\end{equation*}
The preceding inequality ensures that $\beta\le1/2$.

For every suboptimal arm $i\in\{2,\ldots,K\}$, consider an alternative
instance $\underline\nu^{(i)}$ that differs from
$\underline\nu^{(0)}$ only in the distribution of arm $i$, which is
replaced by
\begin{equation*}
    \nu^{(i)}_i
    =
    (1-\beta)\delta_0
    +
    \beta\delta_{\frac{2\Delta}{\beta}}.
\end{equation*}
The mean of the modified arm is
\begin{equation*}
    \mu_i^{(i)}
    =
    \beta\frac{2\Delta}{\beta}
    =
    2\Delta.
\end{equation*}
Therefore, arm $i$ is the unique optimal arm under
$\underline\nu^{(i)}$. Arm $1$ has gap $\Delta$, while every arm
$j\notin\{1,i\}$ has gap $2\Delta$.

The $(1+\epsilon)$-moment of the modified arm is
\begin{align*}
    \mathbb E_{X\sim\nu_i^{(i)}}
    \left[
        |X|^{1+\epsilon}
    \right]
    &=
    \beta
    \left(
        \frac{2\Delta}{\beta}
    \right)^{1+\epsilon}
    \\
    &=
    (2\Delta)^{1+\epsilon}\beta^{-\epsilon}.
\end{align*}
Hence, $\underline\nu^{(i)}$ belongs to
$\mathcal H_{\epsilon,u_i}$ with
\begin{equation*}
    u_i
    \coloneqq
    (2\Delta)^{1+\epsilon}\beta^{-\epsilon},
\end{equation*}
and
\begin{equation*}
    u_i^{\frac{1}{1+\epsilon}}
    =
    2\Delta\beta^{-\rho}.
\end{equation*}

By the moment-free distribution-free guarantee,
\begin{align*}
    R_T(\underline\nu^{(i)})
    &\le
    2\Delta\beta^{-\rho}\Phi
    \\
    &=
    2\Delta
    \frac{T}{16\Phi}
    \Phi
    \\
    &=
    \frac{\Delta T}{8},
\end{align*}
where we used
\begin{equation*}
    \beta^\rho
    =
    \frac{16\Phi}{T}.
\end{equation*}

Every pull of an arm different from $i$ incurs regret at least $\Delta$
under $\underline\nu^{(i)}$. It follows that
\begin{equation*}
    \mathbb E_{\underline\nu^{(i)}}
    \left[
        T-N_i(T)
    \right]
    \le
    \frac{T}{8}.
\end{equation*}

For ease of notation, let
\begin{equation*}
    x_i
    \coloneqq
    \mathbb E_{\underline\nu^{(0)}}[N_i(T)]
\end{equation*}
and
\begin{equation*}
    y_i
    \coloneqq
    \mathbb E_{\underline\nu^{(i)}}
    \left[
        T-N_i(T)
    \right].
\end{equation*}
Thus,
\begin{equation*}
    y_i\le\frac{T}{8}.
\end{equation*}

Let $\mathbb P_0$ and $\mathbb P_i$ denote the distributions of the
complete interaction history under $\underline\nu^{(0)}$ and
$\underline\nu^{(i)}$, respectively. Since the two instances differ
only on arm $i$, the adaptive KL decomposition gives
\begin{align*}
    \mathrm{KL}(\mathbb P_0,\mathbb P_i)
    &=
    x_i
    \mathrm{KL}
    \left(
        \delta_0,
        (1-\beta)\delta_0
        +
        \beta\delta_{\frac{2\Delta}{\beta}}
    \right)
    \\
    &=
    x_i
    \log\left(\frac{1}{1-\beta}\right).
\end{align*}
Since $\beta\le1/2$,
\begin{equation*}
    \log\left(\frac{1}{1-\beta}\right)
    \le
    2\beta,
\end{equation*}
and therefore
\begin{equation*}
    \mathrm{KL}(\mathbb P_0,\mathbb P_i)
    \le
    2\beta x_i.
\end{equation*}

Consider the event
\begin{equation*}
    A_i
    \coloneqq
    \left\{
        N_i(T)>\frac{T}{2}
    \right\}.
\end{equation*}
By the definitions of $x_i$ and $y_i$,
\begin{equation*}
    x_i
    \ge
    \frac{T}{2}\mathbb P_0(A_i)
\end{equation*}
and
\begin{equation*}
    y_i
    \ge
    \frac{T}{2}\mathbb P_i(A_i^c).
\end{equation*}
The Bretagnolle--Huber inequality then gives
\begin{align*}
    x_i+y_i
    &\ge
    \frac{T}{2}
    \left(
        \mathbb P_0(A_i)
        +
        \mathbb P_i(A_i^c)
    \right)
    \\
    &\ge
    \frac{T}{4}
    \exp\left(
        -\mathrm{KL}(\mathbb P_0,\mathbb P_i)
    \right)
    \\
    &\ge
    \frac{T}{4}\exp(-2\beta x_i).
\end{align*}
Since $y_i\le T/8$, we have
\begin{equation*}
    x_i+\frac{T}{8}
    \ge
    \frac{T}{4}\exp(-2\beta x_i).
\end{equation*}

We claim that
\begin{equation*}
    x_i
    \ge
    \frac{1}{32}
    \min\left\{
        T,
        \frac{1}{\beta}
    \right\}.
\end{equation*}
Indeed, suppose by contradiction that
\begin{equation*}
    x_i
    <
    \frac{1}{32}
    \min\left\{
        T,
        \frac{1}{\beta}
    \right\}.
\end{equation*}
Then
\begin{equation*}
    x_i<\frac{T}{32}
    \qquad\text{and}\qquad
    \beta x_i<\frac{1}{32}.
\end{equation*}
Consequently,
\begin{equation*}
    x_i+\frac{T}{8}
    <
    \frac{5T}{32},
\end{equation*}
whereas
\begin{align*}
    \frac{T}{4}\exp(-2\beta x_i)
    &>
    \frac{T}{4}\exp\left(-\frac{1}{16}\right)
    \\
    &>
    \frac{3T}{16}
    =
    \frac{6T}{32}.
\end{align*}
This is a contradiction. Therefore,
\begin{equation*}
    x_i
    \ge
    \frac{1}{32}
    \min\left\{
        T,
        \left(
            \frac{T}{16\Phi}
        \right)^p
    \right\}.
\end{equation*}

Summing over the $K-1$ suboptimal arms yields
\begin{equation*}
    R_T(\underline\nu^{(0)})
    \ge
    \frac{K-1}{32}\Delta
    \min\left\{
        T,
        \left(
            \frac{T}{16\Phi}
        \right)^p
    \right\}.
\end{equation*}

It remains to remove the minimum. Let
$c_{\mathrm{std},\epsilon}>0$ be the constant appearing in
Theorem~\ref{thr:standard_lb}. Since every valid moment-free
distribution-free rate must satisfy the standard minimax lower bound,
for all sufficiently large $T$,
\begin{equation*}
    \Phi
    \ge
    c_{\mathrm{std},\epsilon}
    K^\rho
    T^{\frac{1}{1+\epsilon}}.
\end{equation*}
It follows that
\begin{align*}
    \left(
        \frac{T}{16\Phi}
    \right)^p
    &\le
    \frac{T}{
        (16c_{\mathrm{std},\epsilon})^pK
    }
    \\
    &\le
    \frac{T}{
        2(16c_{\mathrm{std},\epsilon})^p
    },
\end{align*}
where the last inequality uses $K\ge2$.

Define
\begin{equation*}
    \kappa_\epsilon
    \coloneqq
    \min\left\{
        1,
        2(16c_{\mathrm{std},\epsilon})^p
    \right\}.
\end{equation*}
The preceding upper bound implies
\begin{equation*}
    \min\left\{
        T,
        \left(
            \frac{T}{16\Phi}
        \right)^p
    \right\}
    \ge
    \kappa_\epsilon
    \left(
        \frac{T}{16\Phi}
    \right)^p.
\end{equation*}
Therefore,
\begin{equation*}
    R_T(\underline\nu^{(0)})
    \ge
    c_\epsilon
    (K-1)\Delta
    \left(
        \frac{T}{\Phi}
    \right)^p,
\end{equation*}
where
\begin{equation*}
    c_\epsilon
    \coloneqq
    \frac{\kappa_\epsilon}{
        32\,16^p
    }
    =
    \frac{
        \min\left\{
            1,
            2(16c_{\mathrm{std},\epsilon})^{
                \frac{1+\epsilon}{\epsilon}
            }
        \right\}
    }{
        32\,
        16^{\frac{1+\epsilon}{\epsilon}}
    }.
\end{equation*}
Thus, all constants are now explicitly specified in terms of the constant
$c_{\mathrm{std},\epsilon}$ from the standard minimax lower bound.

On the baseline instance,
\begin{equation*}
    \sum_{i:\Delta_i>0}\Delta_i
    =
    (K-1)\Delta.
\end{equation*}
By Definition~\ref{def:phi_dep}, for every $\eta>0$ and all sufficiently
large $T$,
\begin{equation*}
    R_T(\underline\nu^{(0)})
    \le
    (1+\eta)
    \Phi_{dep}(K,T)
    (K-1)\Delta.
\end{equation*}
Combining the upper and lower bounds and cancelling $(K-1)\Delta$ gives
\begin{equation*}
    \frac{
        \Phi_{dep}(K,T)
        \Phi_{free}(K,T)^p
    }{
        T^p
    }
    \ge
    \frac{c_\epsilon}{1+\eta}
\end{equation*}
for all sufficiently large $T$. Taking the inferior limit and then letting
$\eta\rightarrow 0$ yields
\begin{equation*}
    \liminf_{T\rightarrow+\infty}
    \frac{
        \Phi_{dep}(K,T)
        \Phi_{free}(K,T)^{
            \frac{1+\epsilon}{\epsilon}
        }
    }{
        T^{\frac{1+\epsilon}{\epsilon}}
    }
    \ge
    c_\epsilon.
\end{equation*}
This concludes the proof.
\end{proof}
\section{Proofs of the Regret Upper Bounds of \algnameshort in $u$-adaptive Heavy-Tailed Bandits}
We first derive the Median of Means concentration inequality used in both
proofs. The argument is standard and follows from the finite-moment
analysis underlying robust heavy-tailed estimation
\citep{bubeck2013bandits}.
\begin{lemma}[Concentration of the Median of Means estimator]
\label{lem:mom_concentration}
Fix $\epsilon\in(0,1]$, and let $X_1,\ldots,X_n$ be independent samples
with common mean $\mu$ satisfying
\[
    \max_{i \in [n]}\mathbb E[|X_i|^{1+\epsilon}]
    \le u.
\]
Let $1\le B\le n$, divide the samples into $B$ blocks of size
$s=\lfloor n/B\rfloor$, and let $\widehat\mu^{MoM}$ be the Median of Means
estimator defined from these blocks. Then,
\[
    \mathbb P\left(
        \left|
            \widehat\mu^{MoM}-\mu
        \right|
        >
        C_{\epsilon}^{MoM}
        u^{\frac{1}{1+\epsilon}}
        \left(
            \frac{B}{n}
        \right)^{\frac{\epsilon}{1+\epsilon}}
    \right)
    \le
    \exp\left(-\frac{B}{8}\right),
\]
where
\[
    C_{\epsilon}^{MoM}
    \coloneqq
    2^{
        1+\frac{3+\epsilon}{1+\epsilon}
    }.
\]
\end{lemma}
\begin{proof}
Let $p=1+\epsilon$. By Jensen's inequality,
\[
    |\mu|^p
    \le
    \mathbb E[|X_1|^p]
    \le
    u.
\]
Therefore,
\[
    \mathbb E[|X_1-\mu|^p]
    \le
    2^{p-1}
    \left(
        \mathbb E[|X_1|^p]+|\mu|^p
    \right)
    \le
    2^p u.
\]

Let $\overline X_b$ be the average of one block of size
$s=\lfloor n/B\rfloor$. The von Bahr--Esseen inequality gives
\[
    \mathbb E\left[
        |\overline X_b-\mu|^p
    \right]
    \le
    2^{p+1}u\,s^{1-p}.
\]
By Markov's inequality,
\[
    \mathbb P\left(
        |\overline X_b-\mu|
        >
        2^{1+\frac{3}{p}}
        u^{\frac1p}
        s^{-\frac{p-1}{p}}
    \right)
    \le
    \frac14.
\]
Since $s\ge n/(2B)$,
\[
    2^{1+\frac{3}{p}}
    s^{-\frac{p-1}{p}}
    \le
    2^{2+\frac{2}{p}}
    \left(
        \frac{B}{n}
    \right)^{\frac{p-1}{p}}
    =
    C_{\epsilon}^{MoM}
    \left(
        \frac{B}{n}
    \right)^{\frac{\epsilon}{1+\epsilon}}.
\]

If the Median of Means estimator deviates from $\mu$ by more than the
displayed radius, at least half of the block averages must be bad. The
blocks are independent, and each is bad with probability at most $1/4$.
Hoeffding's inequality therefore yields
\[
    \mathbb P\left(
        \left|
            \widehat\mu^{MoM}-\mu
        \right|
        >
        C_{\epsilon}^{MoM}
        u^{\frac{1}{1+\epsilon}}
        \left(
            \frac{B}{n}
        \right)^{\frac{\epsilon}{1+\epsilon}}
    \right)
    \le
    \exp\left(-\frac{B}{8}\right).
\]
\end{proof}
\adaptiveUBfree*

\begin{proof}
Let
\[
    \rho_\epsilon
    \coloneqq
    \frac{\epsilon}{1+\epsilon},
    \qquad
    r_q
    \coloneqq
    \rho_\epsilon(1-q),
\]
and recall that
\[
    \beta_\alpha
    =
    \frac{1-\alpha}{\rho_\epsilon}.
\]
For ease of notation, define
\[
    V
    \coloneqq
    u^{\frac{1}{1+\epsilon}},
    \qquad
    M_T
    \coloneqq
    K^qT^{\beta_\alpha},
    \qquad
    H_T
    \coloneqq
    K^{r_q}T^\alpha.
\]

By Jensen's inequality, $|\mu_i|\le V$ for every arm. Therefore, every
suboptimality gap is bounded by
\[
    \Delta_i
    \le
    2V.
\]

The restrictions on $q$ and $\alpha$ imply
\[
    q
    \le
    \rho_\epsilon(1-q)
    =
    r_q
\]
and
\[
    \beta_\alpha
    \le
    \alpha.
\]
Consequently,
\[
    M_T
    \le
    H_T.
\]
Moreover, $q\le\epsilon/(1+2\epsilon)$ implies
$r_q\ge\epsilon/(1+2\epsilon)$, while
$\alpha\ge(1+\epsilon)/(1+2\epsilon)$. Hence,
\[
    \alpha
    \ge
    1-r_q.
\]

We now bound the exploration length. If $K\le T$, then
\[
    K
    =
    K^{r_q}K^{1-r_q}
    \le
    K^{r_q}T^{1-r_q}
    \le
    H_T,
\]
and therefore
\[
    L_T
    \le
    KB_T+M_T+1
    \le
    (B_T+2)H_T.
\]
If $K>T$, then $\widetilde L_T\ge KB_T>T$, so that $L_T=T$. Since
$r_q+\alpha\ge1$, we also have
\[
    L_T
    =
    T
    \le
    K^{r_q}T^\alpha
    =
    H_T.
\]
Thus, in all cases,
\[
    L_T
    \le
    (B_T+2)H_T.
\]

If $L_T=T$, the algorithm only explores, and the regret satisfies
\[
    R_T(\underline\nu)
    \le
    2VL_T
    \le
    2V(B_T+2)H_T.
\]
We may therefore assume that $L_T<T$. In this case,
$L_T=\widetilde L_T$, and every arm receives at least
\[
    n_i
    \ge
    \left\lfloor
        \frac{L_T}{K}
    \right\rfloor
    \ge
    B_T
\]
samples. Moreover, writing $M_T=K^qT^{\beta_\alpha}$ and using $B_T\ge1$,
\begin{align*}
    n_i
    &\ge
    \left\lfloor
        \frac{KB_T+\lceil M_T\rceil}{K}
    \right\rfloor
    \\
    &=
    B_T+
    \left\lfloor
        \frac{\lceil M_T\rceil}{K}
    \right\rfloor
    \\
    &\ge
    B_T+\frac{M_T}{K}-1
    \ge
    K^{q-1}T^{\beta_\alpha}.
\end{align*}

Let $\mathcal E_T$ be the event on which, simultaneously for every
$i\in[K]$,
\[
    \left|
        \widehat\mu_i^{MoM}-\mu_i
    \right|
    \le
    C_{\epsilon}^{MoM}
    V
    \left(
        \frac{B_T}{n_i}
    \right)^{\rho_\epsilon}.
\]
By Lemma~\ref{lem:mom_concentration} and a union bound,
\[
    \mathbb P(\mathcal E_T^c)
    \le
    K\exp\left(-\frac{B_T}{8}\right).
\]
Since $B_T\ge8\log(KT^3)$,
\[
    \mathbb P(\mathcal E_T^c)
    \le
    \frac{1}{T^3}.
\]

On $\mathcal E_T$, every estimate has error at most
\[
    r_T
    \coloneqq
    C_{\epsilon}^{MoM}
    V
    B_T^{\rho_\epsilon}
    K^{r_q}
    T^{-\rho_\epsilon\beta_\alpha}.
\]
Since $\widehat I^*$ maximizes the estimated mean, its suboptimality gap
satisfies
\[
    \Delta_{\widehat I^*}
    \le
    2r_T.
\]
The expected regret can therefore be bounded as
\[
    R_T(\underline\nu)
    \le
    2VL_T
    +
    2r_TT
    +
    2VT\mathbb P(\mathcal E_T^c).
\]
Using
\[
    1-\rho_\epsilon\beta_\alpha
    =
    \alpha,
\]
we obtain
\[
    2r_TT
    =
    2C_{\epsilon}^{MoM}
    V
    B_T^{\rho_\epsilon}
    H_T.
\]
Combining the previous bounds gives
\[
    R_T(\underline\nu)
    \le
    2V(B_T+2)H_T
    +
    2C_{\epsilon}^{MoM}
    V
    B_T^{\rho_\epsilon}
    H_T
    +
    \frac{2V}{T^2}.
\]
Since $B_T\ge1$, $\rho_\epsilon\le1$, and $H_T\ge1$, it follows that
\[
    R_T(\underline\nu)
    \le
    \left(
        2C_{\epsilon}^{MoM}+8
    \right)
    B_T
    V
    H_T.
\]
Recalling the definitions of $V$ and $H_T$, we conclude that
\[
    R_T(\underline\nu)
    \le
    \left(
        2^{
            2+\frac{3+\epsilon}{1+\epsilon}
        }
        +8
    \right)
    B_T
    u^{\frac{1}{1+\epsilon}}
    K^{\frac{\epsilon}{1+\epsilon}(1-q)}
    T^\alpha.
\]
Since $B_T=\lceil8\log(KT^3)\rceil$, the stated
$\widetilde{\mathcal O}$ guarantee follows.
\end{proof}

\adaptiveUBdep*

\begin{proof}
Let
\[
    \rho_\epsilon
    \coloneqq
    \frac{\epsilon}{1+\epsilon},
    \qquad
    V
    \coloneqq
    u^{\frac{1}{1+\epsilon}}.
\]
If every arm is optimal, the regret is identically zero and the claim is
immediate. We may therefore assume that the instance contains at least one
suboptimal arm, and define
\[
    \Delta_{\min}
    \coloneqq
    \min_{i:\Delta_i>0}\Delta_i.
\]

Since $\alpha<1$, we have $\beta_\alpha>0$. Moreover,
\[
    \beta_\alpha
    \le
    \frac{1+\epsilon}{1+2\epsilon}
    <
    1.
\]
For every fixed $K$, it follows that
\[
    KB_T+K^qT^{\beta_\alpha}
    =
    o(T).
\]
Thus, for all sufficiently large $T$, we have
$L_T=\widetilde L_T<T$.

As in the proof of Theorem~\ref{thr:adaptiveUBfree}, every arm receives at
least
\[
    n_i
    \ge
    K^{q-1}T^{\beta_\alpha}
\]
exploration samples. Define
\[
    r_T
    \coloneqq
    C_{\epsilon}^{MoM}
    V
    \left(
        \frac{
            B_TK^{1-q}
        }{
            T^{\beta_\alpha}
        }
    \right)^{\rho_\epsilon}.
\]
Since $K$ and the instance are fixed, $B_T$ is logarithmic in $T$ and
$\beta_\alpha>0$, so that
\[
    \lim_{T\to+\infty}r_T
    =
    0.
\]
Consequently, for all sufficiently large $T$,
\[
    2r_T
    <
    \Delta_{\min}.
\]

Let $\mathcal E_T$ be the event on which all the MoM estimates differ from
their respective means by at most $r_T$. Lemma~\ref{lem:mom_concentration}
and the choice $B_T\ge8\log(KT^3)$ give
\[
    \mathbb P(\mathcal E_T^c)
    \le
    \frac{1}{T^3}.
\]
On $\mathcal E_T$, no suboptimal arm can maximize the estimated mean.
Therefore, the arm selected during the commitment phase is optimal.

During the round-robin exploration phase, every arm is selected at most
\[
    \left\lceil
        \frac{L_T}{K}
    \right\rceil
    \le
    B_T
    +
    K^{q-1}T^{\beta_\alpha}
    +
    2
\]
times. Hence, the exploration regret is bounded by
\[
    \left(
        B_T
        +
        K^{q-1}T^{\beta_\alpha}
        +
        2
    \right)
    \sum_{i:\Delta_i>0}\Delta_i.
\]
The commitment phase incurs no regret on $\mathcal E_T$. On
$\mathcal E_T^c$, its regret is at most $2VT$. We thus obtain, for every
sufficiently large $T$,
\[
    R_T(\underline\nu)
    \le
    \left(
        B_T
        +
        K^{q-1}T^{\beta_\alpha}
        +
        2
    \right)
    \sum_{i:\Delta_i>0}\Delta_i
    +
    \frac{2V}{T^2}.
\]
Dividing by $K^{q-1}T^{\beta_\alpha}$ gives
\[
    \frac{
        R_T(\underline\nu)
    }{
        K^{q-1}T^{\beta_\alpha}
    }
    \le
    \left(
        1+
        \frac{B_T+2}{
            K^{q-1}T^{\beta_\alpha}
        }
    \right)
    \sum_{i:\Delta_i>0}\Delta_i
    +
    \frac{
        2V
    }{
        K^{q-1}T^{\beta_\alpha+2}
    }.
\]
For every fixed $K$,
\[
    \lim_{T\to+\infty}
    \frac{B_T+2}{
        K^{q-1}T^{\beta_\alpha}
    }
    =
    0
\]
and
\[
    \lim_{T\to+\infty}
    \frac{
        2V
    }{
        K^{q-1}T^{\beta_\alpha+2}
    }
    =
    0.
\]
Taking the superior limit proves
\[
    \limsup_{T\to+\infty}
    \frac{
        R_T(\underline\nu)
    }{
        K^{q-1}T^{\beta_\alpha}
    }
    \le
    \sum_{i:\Delta_i>0}\Delta_i.
\]
\end{proof}
\section{Proofs for $(\epsilon,u)$-adaptivity}
Let the constant from Lemma~\ref{lem:mom_concentration} be
\begin{equation*}
    C_{\epsilon}^{MoM}
    \coloneqq
    2^{1+\frac{3+\epsilon}{1+\epsilon}}
\end{equation*}

\adaptiveUnknownParameters*

\begin{proof}
Let
\begin{equation*}
    \rho_\epsilon
    \coloneqq
    \frac{\epsilon}{1+\epsilon},
    \qquad
    V
    \coloneqq
    u^{\frac{1}{1+\epsilon}},
\end{equation*}
and define
\begin{equation*}
    M_T
    \coloneqq
    K^{1/3}T^{2/3},
    \qquad
    H_T
    \coloneqq
    K^{\frac{2\rho_\epsilon}{3}}
    T^{1-\frac{2\rho_\epsilon}{3}}.
\end{equation*}

By Jensen's inequality, $|\mu_i|\le V$ for every arm, and hence every
suboptimality gap satisfies $\Delta_i\le2V$.

Since $\rho_\epsilon\le1/2$ and $T\ge K$, we have
\begin{equation*}
    \frac{M_T}{H_T}
    =
    \left(
        \frac{K}{T}
    \right)^{\frac{1-2\rho_\epsilon}{3}}
    \le1
\end{equation*}
and
\begin{equation*}
    \frac{K}{H_T}
    =
    \left(
        \frac{K}{T}
    \right)^{1-\frac{2\rho_\epsilon}{3}}
    \le1.
\end{equation*}
It follows that
\begin{equation*}
    L_T
    \le
    KB_T+M_T+1
    \le
    (B_T+2)H_T.
\end{equation*}

If $L_T=T$, the algorithm performs only round-robin exploration, and
therefore
\begin{equation*}
    R_T(\underline\nu)
    \le
    2VL_T
    \le
    2V(B_T+2)H_T.
\end{equation*}
We may thus assume that $L_T<T$. In this case,
$L_T=\widetilde L_T$, $M_T=K^{1/3}T^{2/3}$, and the number $n_i$ of samples collected from each
arm satisfies
\begin{align*}
    n_i
    &\ge
    \left\lfloor
        \frac{KB_T+\lceil M_T\rceil}{K}
    \right\rfloor
    \\
    &=
    B_T+
    \left\lfloor
        \frac{\lceil M_T\rceil}{K}
    \right\rfloor
    \\
    &\ge
    B_T+\frac{M_T}{K}-1
    \ge
    K^{-2/3}T^{2/3}.
\end{align*}
In particular, $n_i\ge B_T$, so that the Median of Means estimator is
well defined.

Let $\mathcal E_T$ be the event on which, simultaneously for every
$i\in[K]$,
\begin{equation*}
    \left|
        \widehat\mu_i^{MoM}-\mu_i
    \right|
    \le
    C_{\epsilon}^{MoM}
    V
    \left(
        \frac{B_T}{n_i}
    \right)^{\rho_\epsilon}.
\end{equation*}
By Lemma~\ref{lem:mom_concentration} and a union bound,
\[
    \mathbb P(\mathcal E_T^c)
    \le
    K\exp\left(-\frac{B_T}{8}\right)
    \le
    \frac{1}{T^3}.
\]
On $\mathcal E_T$, the estimation error of every arm is at most
\begin{equation*}
    r_T
    \coloneqq
    C_{\epsilon}^{MoM}
    V
    B_T^{\rho_\epsilon}
    K^{\frac{2\rho_\epsilon}{3}}
    T^{-\frac{2\rho_\epsilon}{3}}.
\end{equation*}
Since $\widehat I^*$ maximizes the estimated mean, its gap satisfies
\begin{equation*}
    \Delta_{\widehat I^*}
    \le
    2r_T.
\end{equation*}
The expected regret is therefore bounded by
\begin{equation*}
    R_T(\underline\nu)
    \le
    2VL_T
    +
    2r_TT
    +
    2VT\mathbb P(\mathcal E_T^c).
\end{equation*}
Using the preceding bounds, we obtain
\begin{align*}
    R_T(\underline\nu)
    &\le
    2V(B_T+2)H_T
    +
    2C_{\epsilon}^{MoM}
    VB_T^{\rho_\epsilon}H_T
    +
    \frac{2V}{T^2}
    \\
    &\le
    \left(
        8+2C_{\epsilon}^{MoM}
    \right)
    B_TVH_T,
\end{align*}
where we used $B_T\ge1$, $\rho_\epsilon\le1$, and $H_T\ge1$.

Recalling the definitions of $V$ and $H_T$, we conclude that
\begin{equation*}
    R_T(\underline\nu)
    \le
    \left(
        8+
        2^{2+\frac{3+\epsilon}{1+\epsilon}}
    \right)
    B_T
    u^{\frac{1}{1+\epsilon}}
    K^{\frac{2\epsilon}{3(1+\epsilon)}}
    T^{\frac{3+\epsilon}{3(1+\epsilon)}}.
\end{equation*}
This proves the distribution-free guarantee.

We now prove the distribution-dependent claim. If every arm is optimal,
then the regret is identically zero. Otherwise, define
\begin{equation*}
    \Delta_{\min}
    \coloneqq
    \min_{i:\Delta_i>0}\Delta_i.
\end{equation*}
For every fixed $K$,
\begin{equation*}
    KB_T+K^{1/3}T^{2/3}
    =
    o(T),
\end{equation*}
and hence $L_T=\widetilde L_T<T$ for all sufficiently large $T$.

On the event $\mathcal E_T$, every estimate has error at most $r_T$.
Since the instance and $K$ are fixed,
\begin{equation*}
    \lim_{T\to+\infty}r_T=0.
\end{equation*}
Consequently, for all sufficiently large $T$,
\begin{equation*}
    2r_T<\Delta_{\min},
\end{equation*}
and the arm selected during the commitment phase is optimal.

During round-robin exploration, every arm is pulled at most
\begin{equation*}
    \left\lceil
        \frac{L_T}{K}
    \right\rceil
    \le
    B_T+
    K^{-2/3}T^{2/3}
    +2
\end{equation*}
times. The exploration regret is thus at most
\begin{equation*}
    \left(
        B_T+
        K^{-2/3}T^{2/3}
        +2
    \right)
    \sum_{i:\Delta_i>0}\Delta_i.
\end{equation*}
The commitment phase incurs no regret on $\mathcal E_T$ and at most
$2VT$ regret on $\mathcal E_T^c$. Therefore, for all sufficiently large
$T$,
\begin{equation*}
    R_T(\underline\nu)
    \le
    \left(
        B_T+
        K^{-2/3}T^{2/3}
        +2
    \right)
    \sum_{i:\Delta_i>0}\Delta_i
    +
    \frac{2V}{T^2}.
\end{equation*}
Dividing by $K^{-2/3}T^{2/3}$ and taking the superior limit gives
\begin{equation*}
    \limsup_{T\to+\infty}
    \frac{
        R_T(\underline\nu)
    }{
        K^{-2/3}T^{2/3}
    }
    \le
    \sum_{i:\Delta_i>0}\Delta_i,
\end{equation*}
because $B_T=\mathcal O(\log T)$ for every fixed $K$.
\end{proof}

\unknownepsilonfrontier*

\begin{proof}
Let
\begin{equation*}
    a
    \coloneqq
    \frac{\epsilon}{1+\epsilon},
    \qquad
    \Phi'
    \coloneqq
    \Phi_{\epsilon'}(K,T).
\end{equation*}

We first prove that
\[
    \Phi'\ge\frac{K-1}{2}.
\]
Fix $\Delta>0$ and consider the all-zero reference process, in which every
arm deterministically returns zero. For a fixed realization of the
strategy's internal randomness, let
\[
    \tau_j
    =
    \inf\{t\in[T]:I_t=j\},
\]
with $\tau_j=+\infty$ if arm $j$ is never selected, and set
\[
    \sigma_j
    =
    (\tau_j-1)\wedge T.
\]

Let $m$ be the number of arms selected at least once, and denote their
ordered first-visit times by
\[
    r_1<r_2<\cdots<r_m.
\]
Since at least $\ell$ distinct arms must have been visited by round
$r_\ell$, we have $r_\ell\ge\ell$. Moreover, every unvisited arm
contributes $T$ to the sum of the $\sigma_j$. Since $T\ge K$,
\begin{align*}
    \sum_{j=1}^K\sigma_j
    &\ge
    \sum_{\ell=1}^m(r_\ell-1)
    +(K-m)T
    \\
    &\ge
    \frac{m(m-1)}{2}
    +(K-m)K
    \\
    &\ge
    \frac{K(K-1)}{2}.
\end{align*}
Taking expectation with respect to the internal randomness of the strategy,
there exists an arm $j\in[K]$ such that
\[
    \mathbb E_0[\sigma_j]
    \ge
    \frac{K-1}{2}.
\]

Now consider the deterministic instance in which arm $j$ always returns
$\Delta$ and every other arm always returns zero. Until arm $j$ is first
selected, the observed history coincides with the all-zero reference
process. Therefore, the regret on this instance is at least
\[
    \Delta\mathbb E_0[\sigma_j].
\]
Its intrinsic moment scale at order $\epsilon'$ equals $\Delta$, and hence
\[
    \Phi'
    \ge
    \frac{
        \Delta\mathbb E_0[\sigma_j]
    }{
        \Delta
    }
    \ge
    \frac{K-1}{2}.
\]

Fix $\Delta>0$ and consider the baseline instance
$\underline\nu^{(0)}$ defined by
\begin{equation*}
    \nu_1^{(0)}=\delta_\Delta,
    \qquad
    \nu_i^{(0)}=\delta_0,
    \quad i\in\{2,\ldots,K\}.
\end{equation*}
Its intrinsic scale at order $\epsilon'$ is $\Delta$. Hence,
\begin{equation*}
    R_T(\underline\nu^{(0)})
    =
    \Delta
    \sum_{i=2}^K
    \mathbb E_0[N_i(T)]
    \le
    \Delta\Phi'.
\end{equation*}
It follows that
\begin{equation*}
    \mathbb E_0[N_1(T)]
    \ge
    T-\Phi'
    \ge
    \frac{3T}{4}.
\end{equation*}
Moreover, there exists an arm $j\in\{2,\ldots,K\}$ such that
\begin{equation*}
    \mathbb E_0[N_j(T)]
    \le
    \frac{\Phi'}{K-1}.
\end{equation*}

Define
\begin{equation*}
    \beta
    \coloneqq
    \frac{K-1}{128\Phi'}.
\end{equation*}
The preliminary lower bound on $\Phi'$ ensures that
$\beta\le1/64$. Construct an alternative instance
$\underline\nu^{(j)}$ by replacing only arm $j$ with
\begin{equation*}
    \nu_j^{(j)}
    =
    (1-\beta)\delta_0
    +
    \beta\delta_{\frac{2\Delta}{\beta}}.
\end{equation*}
The mean of the modified arm is $2\Delta$, so arm $j$ is optimal and arm
$1$ has gap $\Delta$.

The intrinsic scale of the alternative instance at order $\epsilon$ is
\begin{align*}
    U_\epsilon(\underline\nu^{(j)})
    &=
    \left(
        \beta
        \left(
            \frac{2\Delta}{\beta}
        \right)^{1+\epsilon}
    \right)^{\frac{1}{1+\epsilon}}
    \\
    &=
    2\Delta\beta^{-\frac{\epsilon}{1+\epsilon}}.
\end{align*}
Consequently,
\[
    R_T(\underline\nu^{(j)})
    \le
    2\Delta
    \beta^{-\frac{\epsilon}{1+\epsilon}}
    \Phi_\epsilon(K,T).
\]

Let $\mathbb P_0$ and $\mathbb P_j$ denote the laws of the complete
interaction history under the baseline and alternative instances. By the
adaptive KL decomposition,
\begin{align*}
    \mathrm{KL}(\mathbb P_0,\mathbb P_j)
    &=
    \mathbb E_0[N_j(T)]
    \log\left(\frac{1}{1-\beta}\right)
    \\
    &\le
    2\beta
    \frac{\Phi'}{K-1}
    =
    \frac{1}{64}.
\end{align*}

Since $N_1(T)/T\in[0,1]$, Pinsker's inequality gives
\begin{align*}
    \left|
        \frac{\mathbb E_0[N_1(T)]}{T}
        -
        \frac{\mathbb E_j[N_1(T)]}{T}
    \right|
    &\le
    \sqrt{
        \frac{
            \mathrm{KL}(\mathbb P_0,\mathbb P_j)
        }{2}
    }
    \\
    &\le
    \frac{1}{8\sqrt{2}}
    \le
    \frac18.
\end{align*}
It follows that
\begin{equation*}
    \mathbb E_j[N_1(T)]
    \ge
    \frac{5T}{8}.
\end{equation*}
Every pull of arm $1$ incurs regret $\Delta$ under the alternative
instance, and hence
\begin{equation*}
    R_T(\underline\nu^{(j)})
    \ge
    \frac{5\Delta T}{8}.
\end{equation*}
Combining the upper and lower bounds on the alternative regret yields
\begin{equation*}
    \Phi_\epsilon(K,T)
    \ge
    \frac{5}{16}T\beta^a.
\end{equation*}
Therefore,
\begin{align*}
    \Phi_\epsilon(K,T)
    \Phi_{\epsilon'}(K,T)^a
    &\ge
    \frac{5}{16}
    T
    \left(
        \frac{K-1}{128}
    \right)^a
    \\
    &\ge
    \frac{5}{16\sqrt{128}}
    T(K-1)^a\\
    &\ge\frac{5}{16\sqrt{256}}
    TK^a\\
\end{align*}
where the last two inequalities use $a\le1/2$. The claim follows with
\begin{equation*}
    c_1
    \coloneqq
    \frac{5}{16\sqrt{256}}
    =
    \frac{5}{256\sqrt{2}}.
\end{equation*}
\end{proof}

\uniformepsilonimpossibility*

\begin{proof}
Let
\begin{equation*}
    c_2
    \coloneqq
    \frac{5}{16\sqrt{128}}.
\end{equation*}
If
\begin{equation*}
    \Phi_1(K,T)>\frac{T}{4},
\end{equation*}
then
\begin{equation*}
    \sup_{\epsilon\in(0,1]}
    \Phi_\epsilon(K,T)
    \ge
    \Phi_1(K,T)
    >
    \frac{T}{4}
    \ge
    c_2T.
\end{equation*}

Suppose instead that $\Phi_1(K,T)\le T/4$. For every
$\epsilon\in(0,1]$, Theorem~\ref{thm:unknown_epsilon_frontier} with
$\epsilon'=1$ gives
\begin{align*}
    \Phi_\epsilon(K,T)
    &\ge
    c_2
    T
    (K-1)^{\frac{\epsilon}{1+\epsilon}}
    \Phi_1(K,T)^{-\frac{\epsilon}{1+\epsilon}}
    \\
    &\ge
    c_2
    T
    \left(
        \frac{4(K-1)}{T}
    \right)^{\frac{\epsilon}{1+\epsilon}}.
\end{align*}
Letting $\epsilon$ decrease to zero, the last multiplicative factor
converges to one. Therefore,
\begin{equation*}
    \sup_{\epsilon\in(0,1]}
    \Phi_\epsilon(K,T)
    \ge
    c_2T.
\end{equation*}
This proves the claim.
\end{proof}

\adaptiveCalibration*

\begin{proof}
Let
\begin{equation*}
    \rho
    \coloneqq
    \frac{\epsilon}{1+\epsilon},
    \qquad
    \bar\rho
    \coloneqq
    \frac{\bar\epsilon}{1+\bar\epsilon},
    \qquad
    V
    \coloneqq
    u^{\frac{1}{1+\epsilon}}.
\end{equation*}
The calibration parameters can equivalently be written as
\begin{equation*}
    \bar q
    =
    \frac{\bar\rho}{1+\bar\rho},
    \qquad
    \bar\beta
    =
    \frac{1}{1+\bar\rho}.
\end{equation*}
In particular, $1-\bar q=\bar\beta$ and
$\bar q+\bar\beta=1$.

Define
\begin{equation*}
    M_T
    \coloneqq
    K^{\bar q}T^{\bar\beta},
    \qquad
    H_T
    \coloneqq
    K^{\rho\bar\beta}
    T^{1-\rho\bar\beta},
    \qquad
    G_T
    \coloneqq
    \max\{M_T,H_T\}.
\end{equation*}
Since $T\ge K$ and $\bar q+\bar\beta=1$, we have $M_T\ge K$.
Therefore,
\begin{equation*}
    L_T
    \le
    KB_T+M_T+1
    \le
    (B_T+2)G_T.
\end{equation*}

If $L_T=T$, Jensen's inequality gives $|\mu_i|\le V$ and hence
$\Delta_i\le2V$ for every arm. Thus,
\begin{equation*}
    R_T(\underline\nu)
    \le
    2VL_T
    \le
    2V(B_T+2)G_T.
\end{equation*}

Suppose now that $L_T<T$. In this case,
$L_T=\widetilde L_T(\bar\epsilon)$, and every arm receives at least
\begin{align*}
    n_i
    &\ge
    \left\lfloor
        \frac{
            KB_T+\lceil M_T\rceil
        }{K}
    \right\rfloor
    \\
    &\ge
    B_T+\frac{M_T}{K}-1
    \ge
    \frac{M_T}{K}
\end{align*}
exploration samples.

Let $\mathcal E_T$ be the event on which all the Median of Means estimates
satisfy the concentration bound of
Lemma~\ref{lem:mom_concentration}. Since
$B_T=\lceil8\log(KT^3)\rceil$, a union bound gives
\begin{equation*}
    \mathbb P(\mathcal E_T^c)
    \le
    K\exp\left(-\frac{B_T}{8}\right)
    \le
    \frac{1}{T^3}.
\end{equation*}
On $\mathcal E_T$, every estimate has error at most
\begin{align*}
    r_T
    &\coloneqq
    C_\epsilon^{MoM}
    V
    \left(
        \frac{B_T}{n_i}
    \right)^\rho
    \\
    &\le
    C_\epsilon^{MoM}
    V
    B_T^\rho
    K^{\rho(1-\bar q)}
    T^{-\rho\bar\beta}
    \\
    &=
    C_\epsilon^{MoM}
    V
    B_T^\rho
    K^{\rho\bar\beta}
    T^{-\rho\bar\beta}.
\end{align*}
Since the committed arm maximizes the estimated mean, its gap is at most
$2r_T$. Consequently,
\begin{align*}
    R_T(\underline\nu)
    &\le
    2VL_T
    +
    2r_TT
    +
    2VT\mathbb P(\mathcal E_T^c)
    \\
    &\le
    2V(B_T+2)G_T
    +
    2C_\epsilon^{MoM}
    VB_T^\rho H_T
    +
    \frac{2V}{T^2}.
\end{align*}
Using $B_T\ge1$, $\rho\le1$, and $G_T\ge1$, we obtain the explicit bound
\begin{equation}
    R_T(\underline\nu)
    \le
    \left(
        8+2C_\epsilon^{MoM}
    \right)
    B_TVG_T.
    \label{eq:arbitrary_calibration_master_bound}
\end{equation}

It remains to identify which term defines $G_T$. We have
\begin{equation*}
    \frac{M_T}{H_T}
    =
    \left(
        \frac{K}{T}
    \right)^{
        \frac{\bar\rho-\rho}{1+\bar\rho}
    }.
\end{equation*}
Since $T\ge K$, if $\epsilon\le\bar\epsilon$, then
$\rho\le\bar\rho$ and $M_T\le H_T$. Equation
\eqref{eq:arbitrary_calibration_master_bound} therefore gives
\begin{equation*}
    R_T(\underline\nu)
    \le
    \left(
        8+2C_\epsilon^{MoM}
    \right)
    B_TV
    K^{
        \frac{
            \epsilon(1+\bar\epsilon)
        }{
            (1+\epsilon)(1+2\bar\epsilon)
        }
    }
    T^{
        1-
        \frac{
            \epsilon(1+\bar\epsilon)
        }{
            (1+\epsilon)(1+2\bar\epsilon)
        }
    }.
\end{equation*}
If instead $\epsilon\ge\bar\epsilon$, then
$\rho\ge\bar\rho$ and $M_T\ge H_T$, yielding
\begin{equation*}
    R_T(\underline\nu)
    \le
    \left(
        8+2C_\epsilon^{MoM}
    \right)
    B_TV
    K^{\frac{\bar\epsilon}{1+2\bar\epsilon}}
    T^{\frac{1+\bar\epsilon}{1+2\bar\epsilon}}.
\end{equation*}
This proves the distribution-free guarantee.

We finally prove the distribution-dependent guarantee. If every arm is
optimal, the claim is immediate. Otherwise, define
\begin{equation*}
    \Delta_{\min}
    \coloneqq
    \min_{i:\Delta_i>0}\Delta_i.
\end{equation*}
For every fixed $K$ and fixed $\bar\epsilon>0$,
\begin{equation*}
    KB_T+K^{\bar q}T^{\bar\beta}
    =
    o(T),
\end{equation*}
so $L_T<T$ for all sufficiently large $T$. Moreover, the radius $r_T$
converges to zero for every fixed true $\epsilon>0$. Hence, for all
sufficiently large $T$, $2r_T<\Delta_{\min}$ and the committed arm is
optimal on $\mathcal E_T$.

During round-robin exploration, every arm is selected at most
\begin{equation*}
    B_T
    +
    K^{\bar q-1}T^{\bar\beta}
    +
    2
\end{equation*}
times. Therefore,
\begin{equation*}
    R_T(\underline\nu)
    \le
    \left(
        B_T
        +
        K^{\bar q-1}T^{\bar\beta}
        +
        2
    \right)
    \sum_{i:\Delta_i>0}\Delta_i
    +
    \frac{2V}{T^2}.
\end{equation*}
Dividing by $K^{\bar q-1}T^{\bar\beta}$ and taking the superior limit gives
\begin{equation*}
    \limsup_{T\to+\infty}
    \frac{
        R_T(\underline\nu)
    }{
        K^{\bar q-1}T^{\bar\beta}
    }
    \le
    \sum_{i:\Delta_i>0}\Delta_i.
\end{equation*}
Since
$\bar q-1=-(1+\bar\epsilon)/(1+2\bar\epsilon)$, this is exactly
Equation~\eqref{eq:arbitrary_calibration_dep}.
\end{proof}

\end{document}